\documentclass[pdflatex,sn-mathphys-num]{sn-jnl}

\usepackage{graphicx}%
\usepackage{multirow}%
\usepackage{amsmath,amssymb,amsfonts}%
\usepackage{amsthm}%
\usepackage{mathrsfs}%
\usepackage[title]{appendix}%
\usepackage{xcolor}%
\usepackage{textcomp}%
\usepackage{manyfoot}%
\usepackage{booktabs}%
\usepackage{algorithm}%
\usepackage{algorithmicx}%
\usepackage{algpseudocode}%
\usepackage{listings}%
\usepackage{dsfont}
\usepackage{subcaption}
\usepackage{cancel}
\usepackage{makecell} 
\usepackage{multirow}

\usepackage{tikz}
\usetikzlibrary{arrows.meta,positioning,calc}

\newcommand{\indicator}{\mathds{1}}

\newcommand{\buh}{\hat{\bu}}
\newcommand{\uh}{\hat{u}}

\newcommand{\new}[1]{#1}
\newcommand{\newnew}[1]{#1}
\definecolor{midgreen}{rgb}{0, 0.75, 0}

\newcommand{\yh}{\hat{y}}

\newcommand{\ph}{\hat{p}}

\newcommand{\byh}{\hat{\by}}

\newcommand{\bph}{\hat{\bp}}

\DeclareMathOperator{\SDC}{SDC}
\newcommand{\sdc}{\widetilde{d}}
\DeclareMathOperator{\SDL}{SDL}
\DeclareMathOperator{\IDC}{IDC}

\DeclareMathOperator*{\argmax}{arg\,max}
\DeclareMathOperator*{\argmin}{arg\,min}

\newcommand{\R}{\mathbb{R}}
\newcommand{\RR}{\mathbb{R}}

\newcommand{\bzero}{\boldsymbol{0}}

\newcommand{\calN}{\mathcal{N}}

\newcommand{\calX}{\mathcal{X}}
\newcommand{\calY}{\mathcal{Y}}

\newcommand{\be}{\mathbf{e}}

\newcommand{\bp}{\mathbf{p}}

\newcommand{\bu}{\mathbf{u}}

\newcommand{\bw}{\mathbf{w}}

\newcommand{\by}{\mathbf{y}}

\newcommand{\meanstd}[2]{\makecell[tc]{#1\\[-1ex]{\tiny$\pm$#2}}}

\theoremstyle{thmstyleone}%
\newtheorem{theorem}{Theorem}
\newtheorem{proposition}[theorem]{Proposition}%
\newtheorem{lemma}[theorem]{Lemma}%
\newtheorem{corollary}[theorem]{Corollary}%

\theoremstyle{thmstyletwo}%
\newtheorem{example}{Example}%
\newtheorem{remark}{Remark}%

\theoremstyle{thmstylethree}%

\begin{document}

\title[Article Title]{Soft Dice Confidence: A Near-Optimal Confidence Estimator for Selective Prediction in Semantic Segmentation}


\author[1]{\fnm{Bruno} \sur{L.~C.~Borges}}\email{bruno.laboissiere@posgrad.ufsc.br}

\author[2,3]{\fnm{Bruno} \sur{M.~Pacheco}}\email{bruno.machado.pacheco@umontreal.ca}
\equalcont{These authors contributed equally to this work.}

\author*[1]{\fnm{Danilo} \sur{Silva}}\email{danilo.silva@ufsc.br}
\equalcont{These authors contributed equally to this work.}

\affil[1]{\orgdiv{Department of Electrical and Electronic Engineering}, \orgname{Federal University of Santa Catarina}, \orgaddress{\city{Florianópolis}, \country{Brazil}}}

\affil[2]{\orgdiv{Department of Automation and Systems Engineering}, \orgname{Federal University of Santa Catarina}, \orgaddress{\city{Florianópolis}, \country{Brazil}}}

\affil[3]{\orgdiv{CIRRELT and Département d'Informatique et de Recherche Opérationnelle}, \orgname{Université de Montréal}, \orgaddress{\city{Montréal}, \country{Canada}}}

\abstract{    
In semantic segmentation, even state-of-the-art deep learning models fall short of the performance required in certain high-stakes applications such as medical image analysis. In these cases, performance can be improved by allowing a model to abstain from making predictions when confidence is low, an approach known as selective prediction. While well-known in the classification literature, selective prediction has been underexplored in the context of semantic segmentation. This paper tackles the problem by focusing on image-level abstention, which involves producing a single confidence estimate for the entire image, in contrast to previous approaches that focus on pixel-level uncertainty.
Assuming the Dice coefficient as the evaluation metric for segmentation, two main contributions are provided in this paper: (i) In the case of known marginal posterior probabilities, we derive the optimal confidence estimator, which is observed to be intractable for typical image sizes. Then, an approximation computable in linear time, named Soft Dice Confidence (SDC), is proposed and proven to be tightly bounded to the optimal estimator. (ii) When only an estimate of the marginal posterior probabilities are known, we propose a plug-in version of the SDC and show it outperforms all previous methods, including those requiring additional tuning data. These findings are supported by experimental results on both synthetic data and real-world data from six medical imaging tasks, including out-of-distribution scenarios, positioning the SDC as a reliable and efficient tool for selective prediction in semantic segmentation.
}




\maketitle

\section{Introduction}

Semantic segmentation, especially when powered by deep learning, has become a crucial tool in various domains, particularly in medical imaging. \new{In this setting, the task consists of assigning a semantic label to each individual pixel (or voxel) of an input image, producing a structured output that has the same spatial resolution as the input. Unlike image-level classification, which outputs a single label per image, this dense prediction paradigm captures fine-grained spatial information and object boundaries, making it particularly suitable for medical image analysis, where precise localization of anatomical structures or lesions is required.}
Its applications include
tumor detection, organ delineation, and the identification of pathological regions within medical images \citep{li_medical_2021, bajwa_g1020_2020, codella_skin_2018}, which
are critical not only for diagnostic correctness but also for planning and monitoring treatment.
However, in many applications, deep learning models are not robust enough to completely replace manual segmentation.

Selective prediction~\citep{el-yaniv_foundations_2010,geifman_selective_2017,geifman_bias-reduced_2019,geifman_selectivenet_2019} \new{(also known as reject option \cite{Hendrickx.etal.2024.Machine-Learning-Reject} and learning to reject~\cite{Zhang.etal.2023.Survey-Learning-Reject})} has emerged as a promising approach to increase the reliability of machine learning models while minimizing human intervention.
A selective prediction model either provides a prediction (the prediction is accepted) or abstains from it (the prediction is rejected).
The abstention is guided by a confidence score function, which ideally provides a higher value for predictions that are of higher quality or that are more likely to be correct.
\new{
The confidence value for each prediction is compared to a predefined acceptance threshold, so that
predictions exceeding the threshold can be accepted with limited human oversight, while low-confidence cases are deferred to a human expert. As a result, expert effort is focused on the most uncertain or error-prone instances, enabling the system to preserve overall output quality while substantially reducing the amount of manual review required from a specialist.}

\new{
Most of the literature in selective prediction focuses on classification tasks, in which case it is referred to as selective classification or classification with a reject option.
Since semantic segmentation corresponds to pixel-wise classification, it could be tempting to apply pixel-wise selective classification on a semantic segmentation task, abstaining from or committing to predictions for individual pixels of a given image.
While straightforward, this approach would present significant practical limitations in high-stakes applications.
In particular, rejecting individual pixels typically results in partial segmentations, which would still require a human specialist to analyze the entire image to ensure correctness. In other words, accepting some pixels of an image while rejecting others does not directly translate into reduced human workload.

In contrast, image-level selective prediction operates at the granularity at which human intervention naturally occurs: the entire image. By associating a single confidence score with each predicted segmentation, image-level abstention enables a clear and operationally meaningful trade-off between automation and manual review, 
where accepted predictions may be used with minimal human inspection, while rejected cases are fully delegated to a specialist.
}
These benefits are particularly appealing to medical image segmentation, as the large size and high resolution of medical images 
make human intervention
time-consuming and costly.

While ideal confidence score functions for classification tasks are well known since the foundational work of \citet{chow_1957,chow_1970} (and have been recently expanded upon by \citet{franc_2023}), semantic segmentation tasks have not received the same treatment.
Current literature on the topic mostly focused on adapting classification confidence score functions for segmentation tasks, resulting in pixel-level~\citep{nair_exploring_2020} or component-level~\citep{molchanova_novel_2022} uncertainty estimates, and aggregating such scores into a single value~\citep{kahl_values_2024,jungo_analyzing_2020}.
Furthermore, the performance of these approaches has only been assessed through limited empirical evaluations, lacking theoretical guarantees.

In this paper, we address these gaps by first formalizing the theoretical problem of selective prediction for semantic segmentation, focusing on image-level abstention.
\new{%
Assuming known marginal posterior probabilities, we build on the theoretical foundations of selective classification to derive the ideal confidence estimator for semantic segmentation under the Dice metric.
Observing that the ideal estimator is intractable for real-world input sizes, we propose a linear-time approximation,%
}
termed Soft Dice Confidence (SDC), and prove that it has a tightly bounded difference from the ideal estimator. 

\new{%
In the more practical case when only an estimate of the marginal posterior probabilities is available, we propose a plug-in version of the SDC.
Then, our theoretical findings are corroborated} by a comprehensive set of experiments using synthetic data---where we have full knowledge of the underlying distributions---as well as real-world data from six medical image segmentation tasks.
Importantly, three of the medical imaging tasks are also assessed under distribution shift conditions, reflecting a real-world scenario where training and deployment data may not share the same distribution.
The results show that SDC significantly outperforms existing candidates and remains competitive even
in comparison to confidence estimators that require additional held-out data for tuning.

The code to reproduce our results is available at \url{https://github.com/gama-ufsc/soft_dice_confidence}.




\section{Preliminaries}\label{sec:background}

Let $P$ be a distribution over $\calX \times \calY$, where $\calX$ is the input space and $\calY$ is the label space, with $p_X$ denoting the marginal density over $\calX$. Let $h: \calX \to \calY$ be a predictive model. For a given loss function $\ell : \mathcal{Y}\times \mathcal{Y} \to \R^{+}$, the \emph{risk} of $h$ is defined as
\begin{equation}\label{eq:risk}
R(h) = \mathbb{E}_{x,y \sim P}\left[ \ell(y,h(x)) \right].
\end{equation}

\subsection{Selective prediction}\label{ssec:selective_prediction}

A \textit{selective model} \citep{el-yaniv_foundations_2010,geifman_selective_2017} is a pair $(h, s)$, where $h: \calX \to \calY$ is a predictive model and $s: \calX \to \{0,1\}$ is a \textit{selection function}. Given an input $x \in \calX$, the selective model outputs a prediction $h(x)$ if $s(x) = 1$ (the prediction
is \textit{accepted}), otherwise it abstains (the prediction
is \textit{rejected}).
Formally, we have 
\begin{equation}\label{eq:selective-model}
(h,s)(x) = \begin{cases}
h(x),				&\text{if $s(x) = 1$}, \\
\text{reject}, 	&\text{if $s(x) = 0$}.
\end{cases}
\end{equation}
A selective model's \textit{coverage} $\phi(s) = \mathbb{E}_{x \sim p_X}[s(x)]$ is the probability mass of the accepted predictions, while its \textit{selective risk} 
\begin{equation}
R(h,s) = \mathbb{E}_{x,y \sim P}[\ell(y, h(x)) \mid s(x) = 1]
\end{equation}
is the risk of the predictive model restricted to the accepted predictions.
In particular, the conventional risk equals the selective risk at \textit{full coverage} (when $\phi(s) = 1$).

Let $\indicator[\cdot]$ denote the indicator function. Without loss of generality, we assume $s(x) = \indicator[g(x) \geq t]$, where $g: \calX \rightarrow \RR$ is a \textit{confidence score function} (also known as a \textit{confidence estimator}), which quantifies the model's confidence on each prediction, and $t \in \RR$ is an acceptance threshold.
By varying $t$, it is generally possible to trade off coverage for selective risk, i.e., a lower selective risk can usually (but not necessarily always) be achieved if more predictions are rejected. \new{In the context of a varying threshold, the selective model is denoted by the pair $(h, g)$.}

The trade-off between coverage and selective risk is captured by the \textit{risk-coverage (RC) curve} \citep{el-yaniv_foundations_2010,geifman_selective_2017}, which depicts $R(h,s)$ as a function of $\phi(s)$. \new{An RC curve characterizes the behavior of a selective model $(h,g)$ across all possible operating points, from full coverage (no rejection) to low coverage (aggressive rejection of low-confidence predictions).} 
A widely used scalar metric that summarizes this curve is the \textit{area under the RC curve} (AURC) \citep{geifman_bias-reduced_2019,ding_revisiting_2020}\new{, defined as the integral of the selective risk over the range of achievable coverage values induced by varying the acceptance threshold. Intuitively, the AURC measures the average selective risk over coverages between 0 and 100\%;
thus, it jointly reflects both the quality of the underlying predictive model and the effectiveness of the confidence score function in prioritizing high-quality predictions. Lower AURC values indicate an overall lower selective risk, and thus better selective prediction performance.}

In practice, coverage and selective risk can be evaluated empirically given a test dataset $\{(x^{(i)}, y^{(i)})\}_{i=1}^N$ drawn i.i.d.\ from~$P$, yielding the \textit{empirical coverage} $\hat{\phi}(s) = \frac{1}{N}\sum_{i=1}^N \indicator[s(x^{(i)})]$ and the \textit{empirical selective risk}
\begin{equation}
\label{selective_risk}
\hat{R}(h,s) = \frac{\frac{1}{N}\sum_{i=1}^N\ell\bigl(y^{(i)},h(x^{(i)})\bigr)s(x^{(i)})}{\hat{\phi}(s)}.
\end{equation}

\subsection{Optimal confidence score functions for selective classification}
\label{ssec:optimal-strategies}

Let $\calY$ be finite. Then $h$ is a classifier, and selective prediction reduces to selective classification.
Let $p_Y(\cdot|x)$ denote the probability mass function of the label conditioned on input $x$.
Assume that $h$ is fixed and that we wish to find a confidence score function $g$ and corresponding threshold $t$ minimizing the selective risk for a given coverage.

For all $x \in \calX$, let
\begin{equation}\label{eq:conditional-risk}
r(x) = \sum_{y \in \calY} p_Y(y|x) \ell(y,h(x))
\end{equation}
denote the \textit{conditional risk} and assume that, for any $\beta \in \RR$, the set $\{x \in \calX: r(x) = \beta\}$ has zero probability mass, which is usually the case when $p_X$ is continuous \citep{franc_2023}.

It is shown in \citep{franc_2023} that, for any given target coverage $0 < c \leq 1$, a selection function that minimizes $R(h, s)$ subject to $\phi(s) \geq c$ is given by $s(x) = \indicator[g(x) \geq t]$, where $g: \calX \to \RR$ is any function such that $r(x) < r(x') \implies g(x) > g(x')$ for all $x, x' \in \calX$, and $t \in \RR$ is the largest value such that $\phi(s) \geq c$.

In the special case of the 0-1 loss function, defined as $\ell_\text{0-1}(y, \yh) = \indicator[\yh \neq y]$ for all $y,\yh \in \calY$, the selective risk becomes
\begin{equation}
r(x) = 1 - p_Y(h(x)|x).
\end{equation}
Thus, an optimal confidence score function is given by
\begin{equation}
g(x) = p_Y(h(x) | x)
\end{equation}
which we can interpret as the posterior probability of the predicted class. In particular, for the Bayes-optimal (maximum a posteriori) classifier 
\begin{equation}
h(x) = \argmax_{y \in \calY} p_Y(y|x)
\end{equation}
an optimal confidence score function is given by
\begin{equation}
g(x) = \max_{y \in \calY} p_Y(y|x).
\end{equation}
This optimal choice of $(h, g)$ was first proved by \citet{chow_1957,chow_1970} under a cost-based formulation and is known as Chow's rule.

In general, it is well-known that the conditional risk can be used to derive the Bayes-optimal classifier
\begin{equation}\label{eq:optimal-classifier}
h(x) = \argmin_{\yh \in \calY} \sum_{y \in \calY} p_Y(y|x) \ell(y,\yh),
\end{equation}
however, the result in \citep{franc_2023} holds for any classifier $h$.

\subsection{The plug-in approach}

In practice, the true class posterior distribution $p_Y(y|x)$ is unknown and only an estimate $\ph_Y(y|x)$ can be obtained from training data. A common approach is to substitute $\ph_Y(y|x)$ for $p_Y(y|x)$ in the expression for conditional risk~\eqref{eq:conditional-risk} (as well as in the expression for the optimal classifier~\eqref{eq:optimal-classifier}), a technique referred to as the \textit{plug-in rule}. 

For the 0-1 loss, this results in the \textit{Maximum Class Probability} (MCP) selective classifier given by the MCP classifier
\begin{equation}
h(x) = \argmax_{y \in \calY} \ph_Y(y|x)
\end{equation}
and the MCP confidence estimator
\begin{equation}\label{eq:MCP-confidence-estimator}
g(x) = \max_{y \in \calY} \ph_Y(y|x).
\end{equation}
Since $\ph_Y(y|x)$ is often computed using a softmax function, \eqref{eq:MCP-confidence-estimator} is also known as the \textit{Maximum Softmax Probability} (MSP).

\new{
\subsection{Learning to reject}

Selective classification is historically known as \textit{classification with a reject option} \cite{Hendrickx.etal.2024.Machine-Learning-Reject}, a concept dating back to the seminal work of Chow (1957--1970) \cite{chow_1957,chow_1970}, while the modern terminology \textit{selective classification/prediction} was introduced later in \cite{el-yaniv_foundations_2010,geifman_selective_2017,geifman_bias-reduced_2019,geifman_selectivenet_2019}. The field is also known as \textit{learning to reject} \cite{Zhang.etal.2023.Survey-Learning-Reject}, which emphasizes that the rejector (the conceptual opposite of the selection function) may in general be learned from data. The difference in terminology also reflects a difference in approach: the literature on learning to reject typically follows a cost-based formulation (each rejection incurs a cost), while selective prediction imposes a constraint on either the selective risk or the coverage of the model. However, it is shown in \cite{franc_2023} that all these formulations are essentially equivalent from a mathematical perspective (although, in practice, risk and coverage are arguably easier to interpret and specify than rejection costs \cite{Zhang.etal.2023.Survey-Learning-Reject}).

We further distinguish between two approaches to obtain a rejector: joint training, where the rejector is learned jointly with the prediction model; and a post-training or \textit{post-hoc} approach, where the predictor is first trained and the rejector is subsequently learned or designed. A challenge with joint training is how to fairly compare rejection strategies, since differences in performance may arise from differences in the underlying prediction models rather than from the rejectors themselves. Post-hoc approaches avoid this issue and are attractive in practice when the rejector is simpler than the predictor and retraining is costly. However, they typically require additional \mbox{tuning/calibration} data; for instance, if the rejector is based on posterior probability estimates produced by an overfitted model, then the training data is typically unsuitable for learning the rejector, requiring hold-out data not used for training.
}


\section{Problem statement}\label{sec:problem_statement}

In this paper, we consider the problem of image-level selective prediction for the predictive task of binary semantic segmentation.

In binary semantic segmentation, the input $x \in \calX = \RR^n$ is an image and the target variable $y \in \calY=\{0, 1\}^{n}$ is the corresponding ground truth segmentation mask, where $n$ denotes the total number of pixels (or voxels) in an image.\footnote{For concreteness, we have focused on the task of binary semantic segmentation; however, none of our results require $\calX$ to have the same dimensionality of $\calY$. We thus could consider the more general problem of binary multilabel classification where $\calX$ is arbitrary. However, our theoretical results in Section~\ref{sec:theoretical-results} do require the evaluation metric to be \textit{decomposable}, i.e., expressible as an expectation of individual losses, as in \eqref{eq:risk}. This excludes many common evaluation metrics for multilabel problems such as micro, macro and weighted averages across labels.} 
For convenience, we denote the target variable $y$ as a vector $\by=(y_1,\ldots,y_n) \in \calY$ and we use $p_{Y_i}(1|x) = \sum_{\by \in \{0,1\}^n: y_i=1} p_Y(\by|x)$ to denote the marginal probability that $y_i=1$ conditioned on the input image $x$. The set of all indices $i \in \{1,\ldots,n\}$ such that $y_i=1$ (respectively, $y_i=0$) is referred to as the \textit{foreground} (\textit{background}) of the image. 

To evaluate the performance of a segmentation model $h: \calX \to \calY$, we use the \emph{Dice coefficient}~\citep{dice_1,dice_2,zijdenbosMorphometricAnalysisWhite1994}, which has become the \emph{de facto} standard evaluation metric in semantic segmentation. For all $\by, \byh \in \{0,1\}^n$, the Dice coefficient is defined as
\begin{equation}\label{eq:dice-coeff}
D(\by, \byh) = \frac{2\sum_{i=1}^n y_i \yh_i }{ \sum_{i=1}^n (y_i + \yh_i)}
\end{equation}
while the performance of $h$ is evaluated as $\mathbb{E}_{x,\by \sim P}\left[ D(\by,h(x)) \right]$. Note that $D(\by, \byh) \in [0, 1]$, with perfect segmentation corresponding to a Dice score of 1.

While the definition \eqref{eq:dice-coeff} is ubiquitous in the literature, it is undefined for $D(\bzero, \bzero)$, where $\bzero$ denotes the all-zero vector, as it leads to a 0/0 division. In this paper, we explicitly address this case by defining
\begin{equation}\label{eq:dice-zero-zero}
D(\bzero, \bzero) = 0.
\end{equation}
One technical advantage of the above definition is that $D(\by, \bzero) = D(\bzero, \byh) = 0$ for all $\by, \byh \in \{0,1\}^n$. An in-depth discussion and justification for our definition~\eqref{eq:dice-zero-zero} can be found in Appendix~\ref{app:dice}.

To define the problem of selective prediction in this context, we consider the setup of Section~\ref{ssec:selective_prediction} with the loss function
\begin{equation}\label{eq:dice-error}
\ell(\by, \byh) = 1 - D(\by, \byh)
\end{equation}
which is sometimes known as the Dice distance/dissimilarity/error. In this case, a selective segmentation model $(h,g)$ either outputs a complete predicted segmentation mask $\byh \in \{0,1\}^n$, if $g(x)$ is above a given threshold, or otherwise abstains on the entire image. Note that, if we assume that only abstention cases will have to be manually processed by a specialist to provide a segmentation, then coverage in this context directly translates into saved effort.

We restrict attention to segmentation models 
that are based on some underlying probabilistic model $f:\calX \to [0,1]^n$, where $\bph = (\ph_1,\ldots,\ph_n) = f(x)$ may be interpreted as an estimate of $(p_{Y_1}(1|x), \ldots, p_{Y_n}(1|x))$. In other words, we assume our segmentation model to be of the form $h(x) = T(f(x))$, where $T: [0,1]^n \to \{0,1\}^n$ is some post-processing operation. Let $\byh = (\yh_1,\ldots,\yh_n) = h(x)$. As an example, a typical choice for $T$ is an element-wise thresholding function
\begin{equation}\label{eq:segmentation-thresholding}
\yh_i = \indicator[\ph_i \geq \gamma], \quad i=1,\ldots,n
\end{equation}
where $\gamma \in [0,1]$ is a decision threshold (typically $\gamma=0.5$).

Furthermore, we focus on the \textit{post-hoc} scenario of designing the confidence estimator $g(x)$ after the segmentation model $h(x)$ has been trained, which is assumed to be part of the problem specification and kept fixed. This is in contrast to approaches that require retraining or fine-tuning the segmentation model $h(x)$. Additionally, we focus on confidence estimators that can be computed directly from the probabilistic output $\bph = f(x)$, assuming that the post-processing operation $T$ has been given (which then allows us to compute $h(x)$), without requiring access to any other internal signals or parameters of the segmentation model. 
\new{In particular, this rules out ensemble-based methods (such as Monte Carlo Dropout \cite{Gal.Ghahramani.2016.Dropout-Bayesian-Approximation}) that produce and rely on multiple predictions per input image, unless these predictions are fused into a single probabilistic output $\bph$, which then must be the sole input to the confidence estimator. (In other words, our framework requires $g(x)$ to be only a function of $\bph$, but is entirely agnostic to the process producing $\bph$.)}
Nevertheless, we allow confidence estimators containing hyperparameters that need to be tuned, typically on additional hold-out data that has not been used to train the segmentation model. We refer to such estimators as ``tunable''; otherwise, as ``non-tunable".

\new{
From the \textit{post-hoc} perspective of learning the confidence estimator $g(x)$ given a fixed $h(x)$, the problem can be interpreted as single-label regression (since $g(x) \in \RR$) under an unconventional evaluation metric (the RC curve, or the AURC). A parallel can be drawn to optimizing a binary scoring classifier for the Area Under the Receiver Operating Characteristic (AUROC) curve, which is similarly nondecomposable. Naturally, one may attempt to use a simplified optimization objective, such as training $g(x)$ to predict the Dice coefficient $D(\by, h(x))$ (as done, e.g., in \cite{jungo_analyzing_2020}). In any case, for a realistic evaluation, the amount of tuning data must be limited, e.g., to a fraction of the test data.
}


\subsection{Comparison with previous problem formulations}\label{ssec:previous-work}

\new{To the best of our knowledge, the problem described above has not appeared explicitly in the literature before, except in our preliminary work~\citep{borges2024selective} and implicitly in the concurrent work of \citet{kahl_values_2024}.}
We now briefly discuss related formulations appearing in the literature.

Many papers have proposed and evaluated uncertainty measures for semantic segmentation; these works are relevant since any uncertainty measure can be turned into a confidence score by taking its negative.

Since semantic segmentation corresponds to pixel-wise classification, a common approach is to take a confidence/uncertainty estimator designed for classification and apply it in a pixel-wise fashion to form an uncertainty map with the same dimensionality as the ground truth segmentation \citep{devries_leveraging_2018, shen_assessing_2019, nair_exploring_2020, lambert_fast_2022}. However, without any adaptation, such uncertainty maps can only be directly applied to selective prediction at the \textit{pixel} level, i.e., in this approach, the predictions for a fraction of the pixels of an image would be rejected. It is unclear how such partial segmentations could be useful in practice, since any specialist in charge of completing the rejected pixel predictions would likely still have to analyze the entire image. Component-level or region-level uncertainty \citep{molchanova_novel_2022} may partly alleviate this problem, but does not solve it entirely.
For example, if an image contains multiple rejected components, a specialist may still have to analyze a significant amount of context around a rejected region.
In general, the concept of coverage becomes less meaningful under these formulations, since it is unlikely to directly translate into saved effort.

Although some works have evaluated uncertainty measures at the image level, none of them have considered the risk-coverage tradeoff under the Dice error metric \eqref{eq:dice-error}.
For instance, \citet{jungo_analyzing_2020} focus primarily on failure detection, i.e., distinguishing between successful and failed predictions, and thus evaluate performance using the Area Under the Receiver Operating Characteristic curve (AUROC) \citep{roc_analysis}. This approach differs from ours in two aspects.
First, it requires an arbitrary threshold to classify predictions into successful or failed, thus ignoring differences in segmentation quality among predictions in the same class.
Second, instead of the risk-coverage tradeoff, it evaluates uncertainty estimation quality by the tradeoff between the fraction of successful predictions that are rejected and the fraction of failed predictions that are accepted, a concept which is arguably less meaningful in practice.
\citet{kushibar_layer_2022}, on the other hand, evaluate uncertainty estimation performance using the risk-coverage tradeoff, but they still compute risk using a ``thresholded" version of the Dice metric (namely, $\ell(\by, \byh) = \indicator[D(\by, \byh) < 0.9]$), as in failure detection. Finally, \citet{holder_efficient_2021} consider image-level uncertainty, but as a tool for out-of-distribution detection, a problem which is only superficially related to segmentation quality using the Dice metric.

The only exception of which we are aware is the work of \citet{kahl_values_2024}, which (among other contributions) evaluated image-level uncertainty using the AURC metric based on the Dice coefficient. However, their focus is on benchmarking components of uncertainty estimation methods over several conditions, aiming to draw general conclusions, rather than addressing the design of the confidence estimator when a specific segmentation model is given as part of the problem.

\section{Image-level confidence estimators}\label{sec:image_level_confidence_estimators}

This section introduces all the confidence estimators considered in this work.

\subsection{Existing methods}\label{ssec:existing-methods}

\subsubsection{Non-tunable}

Non-tunable confidence estimators can be expressed as a function of $\bph = f(x)$, as described in Section~\ref{sec:problem_statement}.

The simplest approach is to take a confidence/uncertainty estimator for classification, apply it element-wise to $\bph$ and then aggregate the scores into a scalar value using some aggregation function such as the mean. This is one of the approaches taken by \citet{jungo_analyzing_2020}, who produce an uncertainty map $\bu=(u_1,\ldots,u_n)$ by computing, for each $j \in \{1,\ldots,n\}$, the entropy of $(1-\ph_j, \ph_j)$, defined as
\begin{equation}\label{eq:uncertainty_norm_entropy}
u_j = - (1 - \ph_j) \log_2 (1 - \ph_j) - \ph_j \log_2 \ph_j ,
\end{equation}
and then aggregate these uncertainty scores as
$ 
g(x) = -\frac{1}{n}\sum_{j=1}^n u_j.
$ 
We refer to this confidence estimator as the \emph{average Negative Entropy} (aNE).

An analogous procedure can be applied to the MSP for each $j \in \{1,\ldots,n\}$, defined as
\begin{equation}\label{eq:msp}
\text{MSP}_j= \max\{\hat{p}_j,\, 1 - \hat{p}_j\}
\end{equation}
resulting in the confidence estimator $g(x) = \frac{1}{n}\sum_{j=1}^n \text{MSP}_j$, which we refer to as the \emph{average MSP} (aMSP).

We also consider a more sophisticated aggregation strategy proposed by \citet{holder_efficient_2021} which, when applied to the uncertainty map $\bu$ defined in \eqref{eq:uncertainty_norm_entropy}, results in
\begin{equation}
g(x) = - \frac{\text{Median}(\bu) + \min_{j} u_j}{\max_{j} u_j}
\end{equation}
where $\text{Median}(\bu)$ is the median of $\{u_1,\ldots,u_n\}$. We refer to this confidence estimator as the \textit{Median-Min-Max Confidence} (MMMC).

\subsubsection{Tunable}

Tunable confidence estimators additionally contain hyperparameters that must be tuned on hold-out data. Let this tuning set be denoted as $\{(x^{(i)}, \by^{(i)})\}_{i=1}^N$.

A simple extension of the aNE estimator, as proposed by \citet{kahl_values_2024}, is to aggregate only uncertainty scores above a certain threshold $\tau \in \RR$:
\begin{equation}\label{eq:tla}
g(x) = - \frac{\sum_{j=1}^n u_j \indicator[u_j > \tau]}{\sum_{j=1}^n \indicator[u_j > \tau]}.
\end{equation}
To choose the threshold $\tau$, \citet{kahl_values_2024} propose to first estimate the mean foreground ratio as
$ 
\alpha = \frac{1}{N} \sum_{i=1}^{N} \frac{1}{n} \sum_{j=1}^n \yh^{(i)}_j
$ 
where $(\yh^{(i)}_1, \ldots, \yh^{(i)}_n) = h(x^{(i)})$. Then, $\tau$ is computed as the $(1-\alpha)$-quantile across all the uncertainty scores $\{u^{(i)}_j,\, j=1,\ldots,n,\, i=1,\ldots,N\}$, where $u^{(i)}_j$ denotes the entropy \eqref{eq:uncertainty_norm_entropy} corresponding to the $j$-th element of the $i$-th sample. Note that only unlabeled tuning data $\{x^{(i)}\}_{i=1}^N$ is required for tuning $\tau$.
We refer to \eqref{eq:tla}, with $\tau$ tuned as described above, as the \textit{Threshold-Level Aggregation} (TLA) confidence estimator.

\new{In addition to TLA, \citet{kahl_values_2024} also introduce the \textit{Patch-Level Aggregation} (PLA) approach, which aggregates uncertainty within local spatial neighborhoods. Specifically, this method computes the sum of uncertainties within each patch of extent $s$ along every spatial dimension of the image, and the highest aggregated value is taken as the image-level score. Although the authors fix $s=10$, corresponding to patches of size $10^D$, where $D$ denotes the number of spatial dimensions of image $x$, this parameter can be treated as a tunable hyperparameter that may be optimized using a dedicated tuning set. In this work, we denote by PLA the configuration with $s=10$, as originally proposed, and by PLA* the tunable version, in which $s$ is selected based on tuning performance.}

\new{The \emph{Patch-Level Aggregation} confidence function $g(x)$ is defined as:}
\begin{equation}\label{eq:pla_sum}
\new{g(x) = -\max_{\mathbf{w} \in \mathcal{W}} 
\sum_{j \in \mathbf{w}} u_j,}
\end{equation}
\new{where a patch $\bw \subseteq \{1,\ldots,n\}$ is understood as a subset of corresponding pixel indices, $\mathcal{W}$ denotes the set of all valid patches of size $s$, and $u_j$ represents the entropy value at pixel $j$, $j=1,\ldots,n$.}

We also consider the best-performing tunable confidence estimator proposed by \citet{jungo_analyzing_2020}, referred to as \textit{Aggregation with Automatically-Extracted Features} (AEF), which we describe as follows.

First, the uncertainty map $\bu$ is quantized as $\buh = (\uh_1,\ldots,\uh_n)$, with $\uh_j = \indicator[u_j > \tau]$ for $j=1,\ldots,n$, for some threshold $\tau \in \RR$. The interpretation is that $\buh$ should predict the segmentation errors $\be = (e_1,\ldots,e_n)$, where $e_j = (1-y_j) \yh_j + y_j (1-\yh_j)$, $j = 1,\ldots,n$, i.e., the union of false positives and false negatives. Thus, the threshold $\tau$ is chosen as the value in the grid $\{0.05, 0.10, \ldots, 0.95\}$ that maximizes the average of the Dice coefficient $D(\be, \buh)$ over the tuning set.

Then, a large number of visual features are automatically extracted from $\bu$ and $\buh$. \citet{jungo_analyzing_2020} employed the PyRadiomics package \citep{van_griethuysen_computational_2017} in its default settings, using $\bu$ as the input image and $\buh$ as the segmentation mask, to extract 102 features, including shape, first-order, and other gray-level characteristics.

Finally, the extracted features are used as input to a random forest regressor that predicts the Dice coefficient $D(\by, \byh)$ of the segmentations. The random forest regressor is configured with 10 base learners and trained on the tuning set to minimize the mean squared error.

\subsection{Proposed method: Soft Dice Confidence}\label{ssec:proposed-method}

We now introduce our proposed confidence estimator. For all $\bph \in [0,1]^n$ and $\byh \in \{0,1\}^n$, the \textit{Soft Dice Confidence} (SDC) is defined as
\begin{align}\label{eq:sdc}
\SDC(\bph, \byh) = \frac{2\sum_{j=1}^n \ph_j \yh_j }{ \sum_{j=1}^n (\ph_j + \yh_j)}, \; \text{if $\bph \neq \bzero \lor \byh \neq \bzero$, \quad and \quad $\SDC(\bzero, \bzero) = 0$.}
\end{align}
Note that the SDC is non-tunable (does not require additional data) and very easy to compute. Since $\byh$ is a function of $\bph$ according to \eqref{eq:segmentation-thresholding}, we may also denote the SDC simply as $\SDC(\bph)$.

Below we provide some intuition for definition \eqref{eq:sdc}. Theoretical justification  is provided in Section~\ref{sec:theoretical-results}.


Recall that, in the context of selective prediction, the goal of an image-level confidence estimator is to serve as a predictor of segmentation quality as evaluated by the Dice coefficient \eqref{eq:dice-coeff}. The latter can be rewritten as
\begin{equation}\label{eq:new_dice_coeff}
D(\by, \byh) 
= \frac{
2\sum_{j=1}^n y_j \yh_j
}{
\sum_{j=1}^n (y_j + \yh_j)
}
= \left(
1
+
\frac{
\textcolor{blue}{\overbrace{\textstyle\sum_{j=1}^n y_j(1-\yh_j)}^{\text{FN}}}
+
\textcolor{red}{\overbrace{\textstyle\sum_{j=1}^n (1-y_j) \yh_j }^{\text{FP}}}
}
{
2\textcolor{midgreen}{\underbrace{\textstyle\sum_{j=1}^n y_j\yh_j }_{\text{TP}}}
}
\right)^{-1}
\end{equation}
where \textcolor{midgreen}{\text{TP}}, \textcolor{blue}{\text{FN}}, and \textcolor{red}{\text{FP}} indicate, respectively, the sum of true positive, false negative, and false positive predictions at the pixel level.
From \eqref{eq:new_dice_coeff}, it is clear that the Dice coefficient does not consider the true negative predictions. In other words, it ignores the portion of the background that is correctly classified. 

Now, observe that an expression analogous to \eqref{eq:new_dice_coeff} can be written for the SDC, with $\bph$ replacing $\by$. Take the case of a false negative $y_j(1-\yh_j)$ and suppose that $\ph_j$ is indeed equal to $p_{Y_j}(1|x)$. When $\yh_j = 0$, we would like $\ph_j$ to be as small as possible, to minimize the probability that $y_j=1$. Thus, $\sum_j \ph_j(1-\yh_j)$ serves as a quantification of the expected sum of false negatives. A similar reasoning can be applied to true and false positives. Notably, the SDC does not consider true negatives.

The same cannot be said for the estimators in Section~\ref{ssec:existing-methods} (except possibly the AEF, which is harder to interpret). Consider for instance the MSP. It can be rewritten, assuming $\gamma=0.5$, as
\begin{equation}\label{eq:msp-rewritten}
\text{MSP}_j= (1-\ph_j) (1-\yh_j) + \ph_j \yh_j ,
\end{equation}
from which we can see it works as an estimator of a correct prediction (either a true negative or a true positive). Thus, we expect it to perform poorly, as it appears to be misaligned with the evaluation metric. A similar reasoning can be applied to the other entropy-based measures (note that $1 - u_j/2$ has a concave shape similar to $\text{MSP}_j$, coinciding at $\ph_j \in \{0, 0.5, 1\}$).

The issue is illustrated in Figure~\ref{fig:image_selective_prediction}, where we can see that 
the high-confidence background region causes the mean aggregation of the MSP to yield a high confidence score even in the case of a bad segmentation (in fact, higher than in the example of a good segmentation). In contrast, the SDC, by excluding pixel-level confidence from likely true negatives, produces a score better aligned to the evaluation metric.

\begin{figure}
  \centering
  \includegraphics[width=0.8\textwidth]{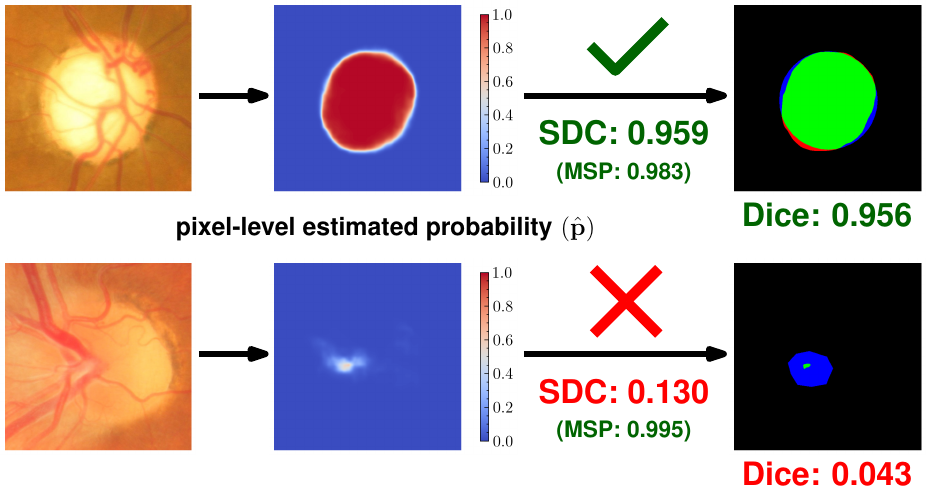}\\[-1ex]
  \caption{Examples of the behavior of two post-hoc confidence estimators for image-level selective prediction. For each example, the input image is shown, followed by the pixel-level estimated probabilities $\bph$, the computed image-level confidence scores, and the final segmentation output. In the predictions, light green, blue and red denote, respectively, \textcolor{midgreen}{true positives}, \textcolor{blue}{false negatives} and \textcolor{red}{false positives}. The images shown are from the ORIGA dataset (more details in Section~\ref{ssec:optic-cup}).}
  \label{fig:image_selective_prediction}
\end{figure}

\begin{remark}
The expression for the SDC resembles the one for the \emph{soft Dice loss} (SDL) \citep{milletari_v-net_2016,drozdzal_importance_2016,sudre_generalised_2017} defined as
\begin{equation}\label{eq:SDL}
\SDL(\bph, \by) = 1 - \frac{2\sum_{j=1}^n \ph_j y_j }{ \sum_{j=1}^n (\ph_j + y_j)}
\end{equation}
where $\by \in \{0,1\}^n$ denotes the ground-truth segmentation.
Specifically, if the hard predictions~$\byh$ are used in place of the ground truth $\by$, we can express
\begin{equation}
\SDC(\bph, \byh) = 1 - \SDL(\bph, \byh).
\end{equation}
Despite their mathematical similarity, the SDL and the SDC serve fundamentally different purposes and have distinct requirements and motivations. 
The SDL is a surrogate loss function designed for gradient-based training and therefore requires ground-truth data to guide model parameter updates. It is often applied at patch and/or batch level and potentially combined with other surrogate loss functions such as the binary cross-entropy loss.
In contrast, the SDC is a confidence estimator intended for ranking predictions in image-level selective prediction and therefore requires only the model output as an estimate of the marginal posterior probabilities; it is derived (see Section~\ref{sec:theoretical-results}) as an approximation to an optimal confidence score function. 
\end{remark}
\new{Since the SDC is designed as a post-hoc confidence estimator to be computed at inference time, its computational properties are particularly relevant in practice. From a computational perspective, the Soft Dice Confidence is highly efficient: its computation scales linearly with the number of pixels, \(O(n)\), and involves only simple vector operations on the model’s probabilistic output. In practice, this additional cost is negligible compared to the computational expense of a single forward pass of a deep segmentation model. In contrast, several existing confidence estimation methods require additional aggregation steps, patch-based processing, or multiple forward passes, resulting in substantially higher inference-time computational costs.}

\section{Theoretical results}
\label{sec:theoretical-results}

Let $\bp = (p_1,\ldots,p_n)$ denote the vector of marginal posteriors $p_j = P_{Y_j}(1|x)$, \mbox{$j=1,\ldots,n$}. In this section, we assume that $\bp$ is known and derive the SDC as an approximation to an optimal confidence score function under the Dice metric. 

\subsection{Conditional risk}


We start with the insight that semantic segmentation can be seen as a special case of classification.
\new{
Each of the $K=2^n$ possible binary segmentation masks can be seen as one out of a set of mutually exclusive classes, although the model's output is evaluated with a non-standard evaluation metric, i.e., not the 0-1 loss or any standard classification metric.

For instance, if $n=4$ (corresponding to a $2 \times 2$ image), then there are exactly $K = 16$ possible binary segmentation masks, which could be treated as a $K$-class classification problem.
Of course, for large $n$ (e.g., for $n=784$, corresponding to a $28\times 28$ image), it would be infeasible to train a conventional classifier
to solve this problem.%
}
Nevertheless, the results of \citet{franc_2023} still hold. As reviewed in Section~\ref{ssec:optimal-strategies}, an optimal confidence score function may be obtained as $g(x) = 1 - r(x)$, where
\begin{equation}\label{eq:conditional-risk-segmentation}
r(x) = \sum_{\by \in \calY} p_{Y}(\by|x) \ell(\by, h(x))
\end{equation}
for any loss function $\ell: \calY \times \calY \to \RR^+$. Therefore, the design of a confidence score function reduces to the problem of efficiently computing the conditional risk \eqref{eq:conditional-risk-segmentation} for the desired loss function.


Let 
$\byh = (\yh_1,\ldots,\yh_n) = h(x)$.
For certain loss functions, the conditional risk can be easily computed.
\begin{proposition}\label{prop:risk-affine-loss}
If $\ell(\by, \byh)$ can be expressed as an affine function of $\by = (y_1,\ldots,y_n)$, i.e., $\ell(\by, \byh) = a_0(\byh) + \sum_{i=1}^n a_i(\byh) y_i$, for some $a_0(\byh), \ldots, a_n(\byh) \in \RR$, then $r(x) = \ell(\bp, \byh)$.
\end{proposition}
\begin{proof}
Since $\sum_{\by \in \calY} p_{Y}(\by|x) y_i = p_i$ for all $i$, the proof follows by linearity of expectation.
\end{proof}

\begin{example}
    The Hamming loss $\ell(\by,\byh) = \frac{1}{n}\sum_{i=1}^{n}\indicator\left[y_i\neq \yh_i\right]$ can be written as $\ell(\by,\byh) = 1 - \frac{1}{n}\sum_{i=1}^{n} \left(y_i \yh_i + (1-y_i)(1-\yh_i)\right)$. Thus, simply replacing $\by$ with $\bp$ in this expression gives us the conditional risk. It follows that an optimal confidence score function is given by 
    \begin{equation}
    g(x) = 1- r(x) = \frac{1}{n}\sum_{i=1}^{n} \left(p_i \yh_i + (1-p_i)(1-\yh_i)\right)
    \end{equation}
    which we can interpret as the average probability of the predicted class.%
    Interestingly, when $\gamma = 0.5$, applying the plug-in rule to this expression results in the aMSP.
\end{example}

However, for general loss functions, computing the conditional risk is not only intractable for large $n$, due to the large cardinality of $\calY$, but actually impossible unless we have access to the full posterior $p_Y(\by|x)$.

\subsection{Ideal Dice confidence}\label{sec:ideal-dice-confidence}

From now on, we focus on the loss function $\ell(\by,\byh)=1-D(\by,\byh)$ induced by the Dice coefficient, as defined in Section \ref{sec:problem_statement}. The conditional risk is given by $r(x)=1-\IDC(x)$ where
\begin{equation}\label{eq:ideal-dice-confidence}
\IDC(x) = \sum_{\by\in \calY} p_{Y}(\by|x) D(\by, h(x))
.\end{equation}
We refer to $\IDC(x)$ as the \textit{ideal Dice confidence} (IDC), as it is an optimal confidence score function.

To enable the computation of the IDC,
we assume that the elements of the target vector are conditionally independent given $x$, i.e., for all $\by \in \{0,1\}^n$,
\begin{equation}\label{eq:conditional-independence}
p_{Y}(\by | x) = \prod_{i=1}^{n} p_{Y_i}(y_i|x) = \prod_{i=1}^{n} p_i^{y_i}(1-p_i)^{1-y_i}.
\end{equation}
Note that, in semantic segmentation, conditional independence is implicitly assumed when training a model with surrogate loss functions such as the cross-entropy loss, and it is also widely accepted due to the intractability of the modeling problem in high dimensions.
For more discussion about these points, see~\citep{dai_rankseg_2023}.

Under this assumption, the IDC can be computed as $\IDC(x) = \IDC(\bp, \byh)$, where\footnote{If conditional independence cannot be assumed, we may make the distinction between the two by referring to $\IDC(x)$ as the \textit{full-posterior} IDC and to (its approximation) $\IDC(\bp, \byh)$ as the \textit{marginal-based} IDC.}
\begin{equation}\label{eq:conditionally-independent-idc}
\IDC(\bp, \byh)
= \sum_{\by\in \mathcal{Y}} \left(  \prod_{i=1}^{n} p_i^{y_i}(1-p_i)^{1-y_i} \right) D(\by, \byh)
,\end{equation}
which now only depends on $\bp$ and $\byh$, thus reducing the space complexity from $O(2^n)$ to $O(n)$. 
However, computing this expression still has complexity $O(2^n)$ and is therefore intractable for large~$n$.

\subsection{Soft Dice confidence}
Although the Dice coefficient is certainly \textit{not} an affine function of $\by$, we can still use the intuition of Proposition~\ref{prop:risk-affine-loss} to propose a candidate approximation to the IDC by simply taking the expression for $D(\by, \byh)$ and replacing $\by$ with $\bp$. This leads precisely to our definition of the SDC in \eqref{eq:sdc}.

Our main theoretical result is a proof that the ratio between $\IDC(\bp, \byh)$ and $\SDC(\bp, \byh)$ is tightly bounded and close to 1.

\begin{theorem}\label{theo:main}
For all $n \geq 1$, all $\bp \in [0,1]^n$, and all $\byh \in \{0,1\}^n$,
\begin{equation}\label{eq:theo-null-statement}
\SDC(\bp,\byh) = 0 \iff  s = 0 \iff \IDC(\bp,\byh)=0
\end{equation}
where $s = \sum_{i=1}^{n} p_i \yh_i$.
Otherwise, if $s > 0$, then
\begin{equation}\label{eq:sdc-bounds-inequalities}
b_L(k,\mu,\lambda) \triangleq \frac{k + k\mu + \lambda}{k + 1 + (k-1)\mu + \lambda} \le \frac{\IDC(\bp,\byh)}{\SDC(\bp,\byh)} \le \sum_{i=0}^{\infty} \frac{\lambda^{i} e^{-\lambda}}{i!} \frac{k + k\mu + \lambda}{k + k\mu + i} \triangleq b_U(k,\mu,\lambda)
\end{equation}
where 
$k =\sum_{i=1}^{n} \yh_i$,
$\mu  = \frac{1}{k} \sum_{i=1}^{n} p_i \yh_i$,
and $\lambda  = \sum_{i=1}^{n} p_i(1-\yh_i)$.
Moreover, 
\begin{equation} \label{eq:bounds-around-1}
b_L(k,\mu,\lambda) \le 1 \le b_U(k,\mu,\lambda).
\end{equation}
\end{theorem}
\begin{proof}
See Appendix~\ref{app:theorem}.
\end{proof}

Note that the quantities $s$, $k$, $\mu$ and $\lambda$ defined in Theorem~\ref{theo:main} can be interpreted, respectively, as the total probability of the predicted foreground, the area/volume of the predicted foreground, the average probability of the predicted foreground, and the total probability of the predicted background.

\begin{corollary}\label{cor:error-bound-eps}
When the IDC is nonzero, the relative error between the SDC and the IDC is upper bounded by
\begin{equation}
\label{eq:relative-error-bound}
\frac{|\SDC(\bp,\byh) - \IDC(\bp,\byh)|}{\IDC(\bp,\byh)} \leq \epsilon(k, \mu, \lambda) \triangleq \max\left\{ \frac{1}{b_L(k,\mu,\lambda)} -1,\, 1 - \frac{1}{b_U(k,\mu,\lambda)} \right\}.
\end{equation}
Moreover,
\begin{equation}
\frac{|\SDC(\bp,\byh) - \IDC(\bp,\byh)|}{\IDC(\bp,\byh)} \leq
\epsilon(s) 
\triangleq \max_{\substack{k,\mu,\lambda:\\ k\mu = s}}\, \epsilon(k, \mu, \lambda) 
= \max\left\{
\epsilon_1(s),\, \epsilon_2(s)
\right\}
\end{equation}
where
\begin{align}
    \epsilon_1(s) &= \max_{\substack{k,\mu,\lambda:\\ k\mu = s}}\, \frac{1}{b_L(k,\mu,\lambda)} -1 = \max_{\mu \in [0,1]}\, \frac{1}{s}		\frac{\mu-\mu^2}{1+\mu} = \frac{3-2\sqrt{2}}{s} \\
\epsilon_2(s) &= \max_{\substack{k,\mu,\lambda:\\ k\mu = s}}\, 1 - \frac{1}{b_U(k,\mu,\lambda)}.
\end{align}
\end{corollary}

Numerical evaluation of $\epsilon(s)$ up to $s \leq 10^4$ shows that, at least for this range, $\epsilon(s) = \epsilon_1(s) > \epsilon_2(s)$. 
Thus, we can see that the relative error quickly decreases with $s$. In particular, the error is below 1\% for $s > 17.16$, which, as an example, could correspond to just 18 high-probability ($\mu \approx 0.95$) or 34 low-probability ($\mu \approx 
0.5$) predicted foreground pixels, regardless of the image dimensions. 
Note also that $\mu$ cannot be too small if $\byh$ is obtained by thresholding $\bp$ as in \eqref{eq:segmentation-thresholding}, since necessarily $\mu > \gamma$. Thus, we expect the approximation to be very good except possibly for images with only a few pixels of predicted foreground.
Evaluations of the tighter bound $\epsilon(k,\mu,\lambda)$ on real-world datasets are shown in Section~\ref{sec:experiments-medical}.

\new{
\subsection{Practical implications}

In practice, the true vector of marginal posteriors $\bp$ is unknown and we only have access to an estimate $\bph = f(x)$. In this case, we can simply replace $\bp$ with $\bph$ in a plug-in fashion to arrive at expression \eqref{eq:sdc}.

In this case, to understand the potential causes of suboptimality of the SDC, it is useful to distinguish between the following concepts:
\begin{itemize}
\item the full-posterior $\IDC(x)$,
\item the marginal-based $\IDC(\bp) = \IDC(\bp, \byh)$,
\item the $\SDC(\bp) = \SDC(\bp, \byh)$,
\item the marginal-based $\IDC(\bph) = \IDC(\bph, \byh)$ under the estimated probabilistic model,
\item and the $\SDC(\bph) = \SDC(\bph, \byh)$ under the estimated probabilistic model,
\end{itemize}
where, in all cases, the same prediction $\byh = h(x) = T(f(x))$ is used, computed by the segmentation model $h(x)$ derived from the estimated probabilistic model $f(x)$.

In general, $\IDC(\bp)$ may be interpreted as an \textit{approximation} to the true optimal confidence estimator $\IDC(x)$ when only the marginal posteriors are known. As illustrated in Fig.~\ref{fig:IDC-SDC}, the two quantities may differ whenever the true distribution deviates from the conditional independence assumption \eqref{eq:conditional-independence}. Similarly, $\IDC(\bph)$ may move away from $\IDC(\bp)$ as the estimated $\bph$ deviates from the true one $\bp$. In practice, however, we may have no way of knowing how close our estimate $\bph$ is to the true distribution, short of simply producing a better approximation. In particular, it may be virtually impossible to check whether conditional independence is violated, as simply storing $p(y|x)$ for a given $x$ requires $2^n$ numbers (not to mention the data requirements to achieve sufficient accuracy in such a fine-grained model). 

\begin{figure}
\centering
\begin{tikzpicture}[
    node distance=6em,
    every node/.style={font=\large},
    arrow/.style={->, thick, shorten <=3pt, shorten >=3pt}
]
\node (idcx) {IDC($x$)};
\node (idcp) [right=of idcx] {IDC($\bp$)};
\node (idcphat) [right=of idcp] {IDC($\bph$)};
\node (sdcp) [below=3ex of idcp] {SDC($\bp$)};
\node (sdcphat) [below=3ex of idcphat] {SDC($\bph$)};
\draw[arrow] (idcx) --
    node[midway, above=2pt, font=\footnotesize, align=center]
    {conditional\\independence}
    (idcp);
\draw[arrow] (idcp) --
    node[midway, above=2pt, font=\footnotesize, align=center]
    {estimated\\model}
    (idcphat);
\node at ($(idcp)!0.5!(sdcp)$) {$\approx$};
\node at ($(idcphat)!0.5!(sdcphat)$) {$\approx$};
\end{tikzpicture}
\vspace{2ex}
\caption{Relationship between versions of IDC and SDC.}
\label{fig:IDC-SDC}
\end{figure}

In contrast, the close relationship between $\IDC(\bp)$ and $\SDC(\bp)$, as well as that between $\IDC(\bph)$ and $\SDC(\bph)$, is established by Theorem~\ref{theo:main}; in particular, Corollary~\ref{cor:error-bound-eps} (when applied to $\bph$) provides a practical way to check whether the relative error between $\SDC(\bph)$ and $\IDC(\bph)$ is guaranteed to be small. It follows that, whenever this bound is confirmed to be sufficiently small, one should intuitively expect the $\SDC(\bph)$ to outperform any other confidence estimator based solely on $\bph$. Further improvements, then, would only be obtained by better (explicit or implicit) probabilistic modeling. For instance, tuning hyperparameters of $g(x)$ on hold-out data (especially when $f(x)$ suffers from overfitting and, presumably, overconfidence) may be a way of implicitly improving the probabilistic model.
}

\section{Experiments with simulated data}\label{sec:synth}

To evaluate the quality of the SDC as an approximation to the IDC, we conduct a series of synthetic experiments where the full posterior $p_Y(\by|x)$ is known.
We use a small $n$ to ensure that the computation of the IDC is feasible. 
Specifically, we use $n=10$.

Ideally, we would like to follow the assumption of conditional independence in \eqref{eq:conditional-independence}. However, this would necessarily imply a nonzero probability for a target image with an empty foreground, i.e., $p_Y(\bzero|x) > 0$, which would result in a less realistic segmentation problem, as discussed in Appendix~\ref{app:dice}, especially for small $n$. Instead, we choose a distribution that approximately satisfies this assumption but excludes the possibility of $\by = \bzero$. Specifically, we choose
\begin{equation}\label{eq:exact-full-posterior}
p_{Y}(\by | x) =
\begin{cases}
0, &\text{if $\by=\bzero$}, \\
\displaystyle\frac{\prod_{j=1}^{n} q_j^{y_j}(1-q_j)^{1-y_j}}{1 - \prod_{j=1}^{n}(1-q_j)}, &\text{otherwise}
\end{cases}
\end{equation}
where $q_i \in [0,1]$, $i=1,\ldots,n$, are dependent on $x$.
It follows that the marginal posterior probabilities are given by
\begin{equation}\label{eq:exact-marginals}
p_i = p_{Y_i}(1|x)  = q_i \frac{1}{1 - \prod_{j=1}^{n}(1-q_j)},\quad i=1,\ldots,n.
\end{equation}
Note that violating conditional independence gives us an opportunity to evaluate the difference between the full-posterior IDC \eqref{eq:ideal-dice-confidence} and its marginal-based version \eqref{eq:conditionally-independent-idc}, although, since \eqref{eq:conditional-independence} is approximately satisfied, we expect the difference to be small.

In the following experiments, the marginal posterior probabilities for each image are generated by independently sampling from a logit-normal distribution for each pixel. Specifically, for each $x$ and each $i=1,\ldots,n$, we generate an independent sample $z_i \sim \calN(\mu_z, \rho_z^2)$ and set $q_i = \sigma(z_i)$, where $\sigma$ is the logistic sigmoid function, $\rho_z=5$, and $\mu_z$ is chosen to produce a desired value for the expected foreground ratio $\alpha = E_{(x,\by)\sim P}[\frac{1}{n}\sum_{i=1}^n y_i]$.
Unless otherwise mentioned, we choose $\alpha \approx 0.25$ (obtained with $\mu_z=-3.698$ for our choice of $n$ and $\rho_z$), which is in line with the highest value observed in our experiments with medical imaging datasets (see Table~\ref{tab:dataset_info}, Section~\ref{sec:datasets}). Note that, under the conditional distribution just described, the dimension of $x$ is immaterial, so we set $\calX = \{1,\ldots,N\}$ and sample $x \in \calX$ uniformly at random.

\subsection{Ideal probabilistic model}\label{sec:synth-ideal-model}

Recall from Section~\ref{sec:problem_statement} that we consider a segmentation model $h: \calX \to \calY$ with output computed according to \eqref{eq:segmentation-thresholding}, based on a probabilistic model $f: \calX \to [0,1]^n$.

First, we assume that this probabilistic model is ideal, i.e., 
\begin{equation}
\bph = f(x) = (p_{Y_1}(1|x),\ldots,p_{Y_n}(1|x)) = \bp.
\end{equation}

Because of the small size of our target space, we are able to compute the true RC curve, as defined in Section~\ref{ssec:selective_prediction}.
We consider a dataset with $N = 5000$ input images, and compute the true conditional risk $r(x)$ for every image.
In turn, this allows us to easily compute the selective risk for every possible coverage level given different confidence score functions.
We consider four confidence score functions: the SDC; the full-posterior $\IDC(x)$ and its marginal-based version $\IDC(\bph) = \IDC(\bph, \byh)$; and the aMSP, which is a natural choice for confidence score function (see Section \ref{sec:image_level_confidence_estimators}).

We repeat our experiments 10 times, with different randomly generated distributions, and measure the performance through the AURC (see Section~\ref{ssec:selective_prediction}). The results are illustrated in Figure~\ref{fig:true-rc-ideal-model}.
It is noteworthy how the true RC curve of the SDC is very close to that of the marginal-based IDC \emph{and} that of the full-posterior IDC.
In fact, the AURC of the SDC is less than 1\% larger than that of the full-posterior IDC, whereas, for comparison, the AURC of the average MSP is 17\% larger.

The performance of the SDC is, in part, justified by how well the SDC approximates the full-posterior IDC, as illustrated in Figure~\ref{fig:sdc-idc-relative-error}.
\new{%
We note that this approximation is a direct result of our Theorem~\ref{theo:main}.
More specifically, Theorem~\ref{theo:main} guarantess that the SDC approximates well the marginal-based IDC.
In turn, the quasi-conditional independence of the posterior probability distribution implies that the marginal-based IDC will be close to the full-posterior IDC.
}

\begin{figure}
    \centering
    \begin{subfigure}{0.49\textwidth}
        \includegraphics[width=\textwidth]{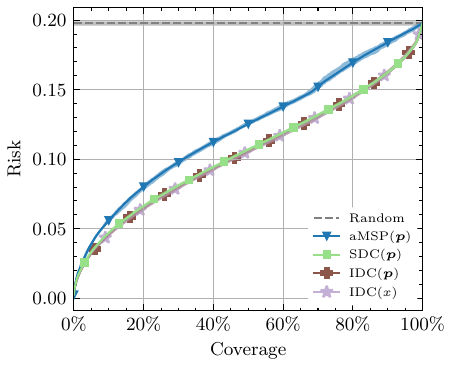}
        \caption{Ideal Segmentation Model.}\label{fig:true-rc-ideal-model}
    \end{subfigure}
    \hfill
    \begin{subfigure}{0.49\textwidth}
        \includegraphics[width=\textwidth]{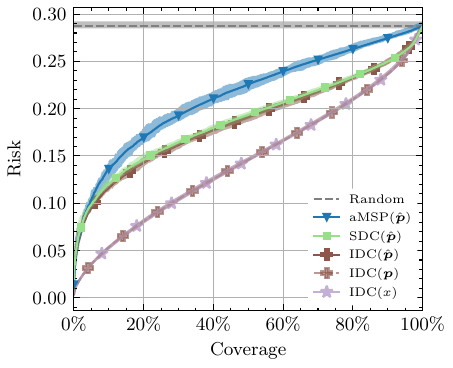}
        \caption{Non-ideal Segmentation Model.}\label{fig:true-rc-noisy-model}
    \end{subfigure}
    \caption{%
        RC curves of a semantic segmentation model with different confidence score functions on synthetic data with expected foreground ratio of 25\%.
        The experiments are repeated 10 times, each with 5000 samples.
        The shaded regions denotes the maximum and minimum values over all runs of the experiments, while the curves are average values.
        In both plots, $\IDC(\bp)$ indicates the marginal-based IDC, following \eqref{eq:conditionally-independent-idc}, whereas $\IDC(x)$ indicates the full-posterior IDC, following \eqref{eq:ideal-dice-confidence}.
        In \ref{fig:true-rc-ideal-model}, the model has access to the true marginal posterior probabilities $\bp$. In \ref{fig:true-rc-noisy-model}, the model has a perturbed version $\bph$ of the same probabilities.
        \emph{Random} indicates the expected risk of a confidence score function that abstains at random.
        The argument of each confidence score function indicates whether it was computed with the true vector of marginal probabilities ($\bp$) or with a perturbed vector ($\bph$).
    }
    \label{fig:rc-synthetic}
\end{figure}

\begin{figure}
    \centering
    \includegraphics{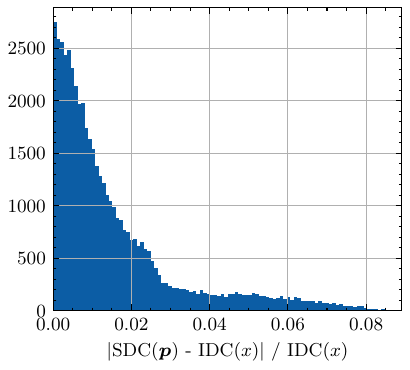}
    \caption{Relative error of the SDC to the full-posterior IDC on all images of our synthetic experiments with the ideal probabilistic model.}
    \label{fig:sdc-idc-relative-error}
\end{figure}

Up to now, we have set $\mu_z$ such that the target image has an expected foreground ratio $\alpha \approx 0.25$, as previously discussed.
Although low foreground ratio values are to be expected in practice, it is unclear how that trait affects the performance of confidence score functions.
We perform further experiments with the ideal model by varying the target's expected foreground ratio through $\mu_z$.
For each value of $\mu_z$, we generate a new population, and compute the true RC curve for the same confidence score functions as before.
We use the AURC as a summarizing metric, taking into account the effects at all coverage values. 
The results are illustrated in Figure~\ref{fig:varying-prevalence}.
Curiously, all confidence score functions converge to a low AURC for extremely low prevalence values.
Nevertheless, the SDC consistently matches the ideal Dice score functions, significantly outperforming the aMSP.

\begin{figure}
    \centering
    \includegraphics{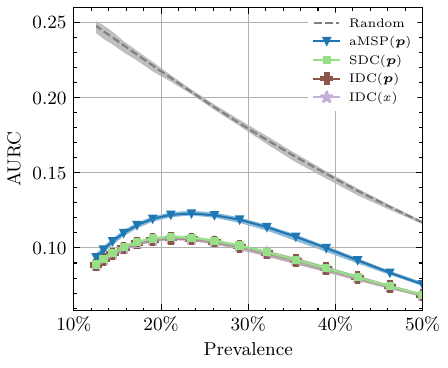}
    \caption{%
    Image-level selective segmentation performance of ideal probabilistic models for different levels of expected foreground ratio of the target image.
    The shaded region contains maximum and minimum values over 10 repetitions of the experiment, while the curves are the average values.
    Each point in the chart is generated by varying $\mu_z$ (expected mean of the marginal probabilities of the target), generating a new dataset, and computing the AURC associated to each confidence score function (see Fig.~\ref{fig:rc-synthetic}).
    }
    \label{fig:varying-prevalence}
\end{figure}

\subsection{Non-ideal probabilistic model}\label{sec:non-ideal-model}


We now consider a more realistic scenario where the model has a good, but non-ideal, estimate of the true marginal posterior probabilities.
We simulate such behavior by adding a logit-normal perturbation to the sampled probability vector $\bp=(p_1,\ldots,p_n)$, that is, a perturbation at the logit level of the probability values.
More precisely, for each output dimension $i$, the perturbed marginal probability vector $\bph$ is defined as
\begin{equation}
    \hat{p}_i = \sigma(\sigma^{-1}(p_i) + \eta_i),\, i=1,\ldots,n,
\end{equation}
where $\eta_i\sim \mathcal{N}(0,\rho_\eta^2)$ is a normally-distributed perturbation.
Note that the perturbation affects not only the confidence estimation but also the segmentation $h(x)$ through \eqref{eq:segmentation-thresholding}, resulting in higher risk.

As previously, we compute true RC curves, which are illustrated in Figure~\ref{fig:true-rc-noisy-model}, with a perturbation with $\rho_\eta=2$.
One addition is the computation of the marginal-based IDC with both the true vector of marginal probabilities $\bp$ and the perturbed one $\bph$.
Note that both $\IDC(x)$ and $\IDC(\bp) = \IDC(\bp, \byh)$ are computed based on the same segmentation model $h(x)$ obtained with non-ideal probabilistic model $\bph$; in other words, the segmentation model is the same for all confidence estimators.
While the results show that the SDC is still clearly better than the aMSP, it is now significantly worse than the IDC computed with the true marginal probabilities.
However, the SDC shows equivalent performance to the IDC when computed with the perturbed vector of marginal probabilities, which, under the assumption that the true marginal probabilities are unknown, is a more realistic comparison.
We note that this is also a direct result of Theorem~\ref{theo:main}, 
as the error bounds hold for any $\bph$.

To evaluate the impact of the probabilistic model's performance on the confidence score functions, we experiment with varying levels of perturbation by controlling $\rho_\eta$.
This way, we can directly estimate how susceptible the confidence score functions are to the quality for the marginal posterior probability estimate.
Starting from $\rho_\eta=0.0$, which is the same as the ideal probabilistic model, we progressively increase the standard deviation of the logit-normal perturbation.
For each value, we generate a new dataset and measure the risk of the resulting model and the AURC derived from each confidence score function.
The results are presented in Figure~\ref{fig:varying-perturbation}.
We note that all confidence score functions present worse performance with decreased model performance (increased risk).
However, the SDC maintains the behavior observed in Figure~\ref{fig:true-rc-noisy-model}, that is, it performs as well as the IDC when computed with the perturbed marginal probabilities.

\begin{figure}
    \centering
    \includegraphics{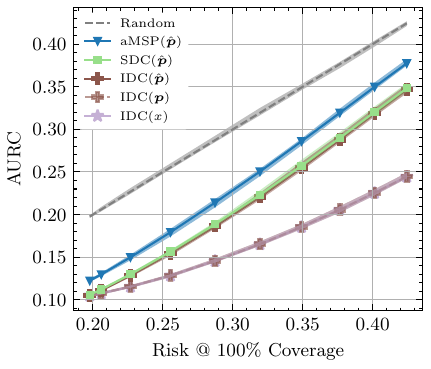}
    \caption{%
    Image-level selective segmentation performance of non-ideal probabilistic models with increasing perturbation levels.
    The shaded region for each curve represents the minimum and maximum values observed over 10 repetitions of the experiments, while the curves are the average values.
    Each point in the chart is generated by adding a logit-normal perturbation (with increasing intensity) to the true marginal probability distribution and using the result to build a segmentation model.
    We note that the starting (left-most) point is generated from not adding any perturbation and, thus, $\bp=\bph$.
    }
    \label{fig:varying-perturbation}
\end{figure}

\section{Experiments with Medical Imaging Data}
\label{sec:experiments-medical}

In this section, we evaluate our proposed confidence score function and baselines on various medical image segmentation tasks, using both in-distribution (ID) and out-of-distribution (OOD) datasets. The latter allow us to assess how well the confidence estimators generalize to data that deviates from the training distribution, an important aspect for real-world deployment in clinical settings.

\subsection{Tasks}\label{sec:datasets}

We consider six medical image segmentation tasks across a broad range of clinical challenges, imaging modalities, and lesion characteristics. 
For each task, we employ a high-performing deep learning model obtained from the literature (in the form of trained weights or original training code). The tasks, models, and datasets used are summarized in Table~\ref{tab:tasks}.


\begin{table}
\centering
\renewcommand{\arraystretch}{1.15}
\caption{Segmentation tasks, models and datasets considered in our evaluation.}
\begin{tabular}{@{} >{\raggedright}p{4cm} p{2cm} >{\raggedright}p{2.5cm} p{3cm} @{}}
\toprule
\textbf{Segmentation Task} & \textbf{Model} & \textbf{ID Dataset} & \textbf{OOD Dataset} \\
\midrule
Brain Tumor & 3D nnUNet & BraTS 2020 & -- \\
Breast Cancer & UNeXt & BUSI & -- \\
Skin Cancer & UNeXt & ISIC 2018 & -- \\
Multiple Sclerosis White Matter Lesion (MS-WML) & 3D UNet & ISBI + MSSEG-1 & PubMRI \\
Optic Cup & Segtran & REFUGE & ORIGA + G1020 \\
Polyp & Polyp-PVT & Kvasir-SEG + ClinicDB & ETIS + ColonDB + EndoScene \\
\bottomrule
\end{tabular}
\label{tab:tasks}
\end{table}

\new{Our ID datasets are exactly those used by the original authors to train and develop the models.
For training and evaluation of segmentation models, as well as the evaluation of non-tunable confidence estimators, the ID dataset is partitioned into training and test sets, while the OOD dataset is used entirely as a test set. These partitions are such that no test data is used as part of model training or development, avoiding data leakage. The size of each test set is presented in Table~\ref{tab:dataset_info}. Additionally, for the evaluation of tunable confidence estimators, each test set is then randomly partitioned into tuning and test sets with varying proportions, as explained in Section~\ref{ssec:experiments-tunable}.

Further details, including task description, training procedures, and dataset distributions are available in Appendix~\ref{app:tasks}.}


\begin{table}
\caption{Descriptive information about each test set. The foreground ratio is the proportion of foreground pixels/voxels across all images in the dataset.}
\centering
\renewcommand{\arraystretch}{1.15}
\begin{tabular}{@{} p{3cm} cccc @{}}
\toprule
\makecell[l]{\textbf{Segmentation Task}} & \makecell[c]{\textbf{Image Size}\\ \textbf{(pixels or voxels)}} & \textbf{Test Set} & \textbf{Size} & \makecell[c]{\textbf{Foreground}\\ \textbf{Ratio}} \\
\midrule
Brain Tumor & 155 $\times$ 240 $\times$ 240 & ID & \new{73} & \new{0.011} \\
Breast Cancer & $256 \times 256$ & ID & 98 & 0.094 \\
Skin Cancer & $512 \times 512$ & ID & 1000 & 0.238 \\
\multirow{2}{*}{MS-WML} & \multirow{2}{*}{\makecell{(154--212)${}\times{}$(212--256)\\${}\times{}$(151--270)}} & ID & 33 & 0.001 \\
 &  & OOD & 25 & 0.002 \\
\multirow{2}{*}{Optic Cup} & \multirow{2}{*}{$575 \times 575$} & ID & 240 & 0.045 \\
 &  & OOD & 1430 & 0.105 \\
\multirow{2}{*}{Polyp} & \multirow{2}{*}{$352 \times 352$} & ID & 164 & 0.123 \\
 &  & OOD & 636 & 0.0616 \\ 
 \bottomrule
\end{tabular}
\label{tab:dataset_info}
\end{table}

\subsection{Empirical Bounds}

In the synthetic experiments of Section~\ref{sec:synth}, we have observed that the risk-coverage performance achieved by $\SDC(\bph)$ is almost identical to that of $\IDC(\bph)$, which is optimal under the assumption that the available probabilistic model is correct. While it is computationally infeasible to compute $\IDC(\bph)$ for real-world data, Theorem~\ref{theo:main} provides a bound on their relative error that can be easily computed.
Table~\ref{tab:result_bounds} shows statistics on $\epsilon(k,\mu,\lambda)$ across all tasks, where we can see that it is at most $4.4 \times 10^{-3}$ and often much smaller. Thus, $\SDC(\bph)$ is effectively indistinguishable from $\IDC(\bph)$, confirming the theoretical insights that form the basis of our proposed estimator.


\begin{table}
\caption{Summary statistics of the bound $\epsilon(k, \mu, \lambda)$ on the relative error between $\SDC(\bph)$ and $\IDC(\bph)$ across all tested datasets, including both ID and OOD data.}
\centering
\renewcommand{\arraystretch}{1.15}
\begin{tabular}{lccc}
\toprule
\textbf{Segmentation Task} & \textbf{Test Set} & \textbf{Max($\epsilon$)} & \textbf{Mean($\epsilon$)} \\ 
\midrule
Brain Tumor & ID & \new{$2.1 \times 10^{-6}$} & \new{$2.5 \times 10^{-7}$} \\ 
Breast Cancer & ID & $2.9 \times 10^{-4}$ & $1.8 \times 10^{-5}$ \\ 
Skin Cancer & ID & $1.2 \times 10^{-4}$ & $1.0 \times 10^{-6}$ \\ 
\multirow{2}{*}{MSWML} & ID & $1.8 \times 10^{-4}$ & $2.0 \times 10^{-5}$ \\
 & OOD & $1.4 \times 10^{-3}$ & $1.2 \times 10^{-4}$ \\
\multirow{2}{*}{Optic Cup} & ID & $2.7 \times 10^{-5}$ & $2.5 \times 10^{-6}$ \\ 
 & OOD & $4.8 \times 10^{-4}$ & $4.4 \times 10^{-6}$ \\ 
\multirow{2}{*}{Polyp} & ID & $2.7 \times 10^{-5}$ & $6.8 \times 10^{-7}$ \\
 & OOD & $4.4 \times 10^{-3}$ & $1.5 \times 10^{-5}$ \\
\bottomrule
\end{tabular}
\label{tab:result_bounds}
\end{table}

\subsection{Experiments with non-tunable confidence estimator}


We first evaluate our proposed confidence estimator against several non-tunable baselines: aMSP, aNE, MMMC, TLA, and \new{PLA}. Although TLA is technically a tunable confidence estimator, it only requires the estimation of a single parameter $\tau$, which can be done entirely using unlabeled data. To simplify the analysis and make a generous evaluation of TLA, we used a fixed $\tau$ estimated using the entire test set (for each test set), effectively making it non-tunable. 

Figure~\ref{fig:result_rc_curves}
\begin{figure}[!ht]
    \centering

    \begin{subfigure}[b]{0.4\textwidth} 
        \includegraphics[width=\textwidth]{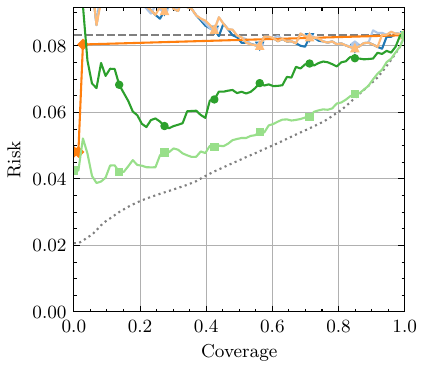}
        \caption{Brain Tumor}
    \end{subfigure}
    \hspace{0.05\textwidth} 
    \begin{subfigure}[b]{0.4\textwidth} 
        \includegraphics[width=\textwidth]{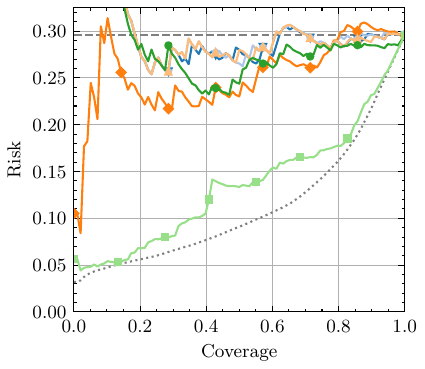}
        \caption{Breast Cancer}
    \end{subfigure}

    \vspace{0.2cm} 
    \begin{subfigure}[b]{0.4\textwidth}
        \includegraphics[width=\textwidth]{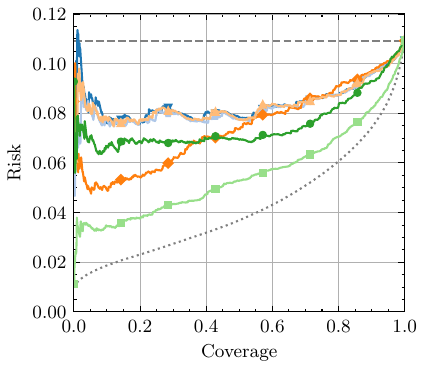}
        \caption{Skin Cancer}
    \end{subfigure}
    \hspace{0.05\textwidth}
    \begin{subfigure}[b]{0.4\textwidth}
        \includegraphics[width=\textwidth]{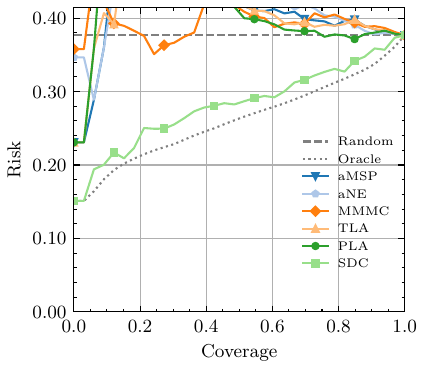}
        \caption{MSWML}
    \end{subfigure}

    \vspace{0.2cm} 
    \begin{subfigure}[b]{0.4\textwidth}
        \includegraphics[width=\textwidth]{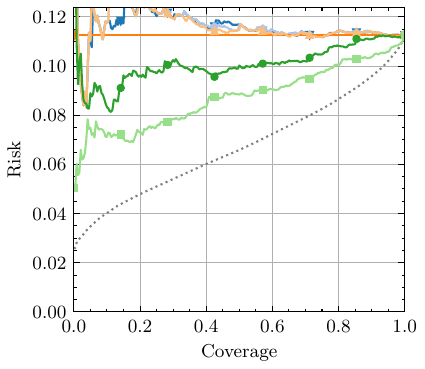}
        \caption{Optic Cup}
    \end{subfigure}
    \hspace{0.05\textwidth} 
    \begin{subfigure}[b]{0.4\textwidth}
        \includegraphics[width=\textwidth]{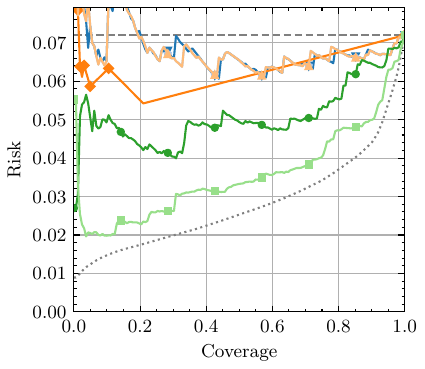}
        \caption{Polyp}
    \end{subfigure}

    \caption{Risk-coverage curves on ID datasets using non-tunable confidence estimators.}
    \label{fig:result_rc_curves}
\end{figure}%
shows RC curves for these confidence estimators across all tasks on ID data.
For reference, we include the expected performance of a random ordering of predictions (denoted ``Random''), which is always equal to the full-coverage risk.
Furthermore, we also include an ``Oracle'' bound, corresponding to a confidence estimator that
always perfectly orders predictions according to the evaluation metric. (Note that this bound is not necessarily achievable, since it requires access to the ground-truth labels.)
As can be seen, SDC consistently outperforms all non-tunable baselines and achieves results close to the Oracle.


Figure~\ref{fig:result_rc_curves_OOD} 
\begin{figure}
    \centering

    \begin{subfigure}[b]{0.4\textwidth} 
        \includegraphics[width=\textwidth]{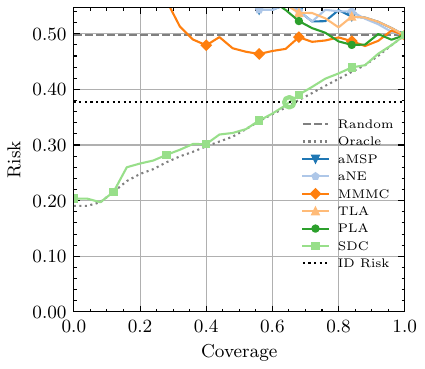}
        \caption{MSWML}
    \end{subfigure}
    \hspace{0.05\textwidth} 
    \begin{subfigure}[b]{0.4\textwidth} 
        \includegraphics[width=\textwidth]{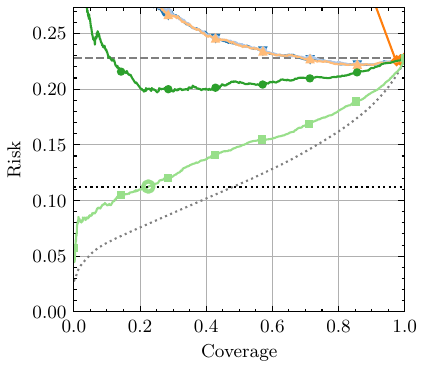}
        \caption{Optic Cup}
    \end{subfigure}
    
    \vspace{0.2cm} 
    \begin{subfigure}[b]{0.4\textwidth}
        \includegraphics[width=\textwidth]{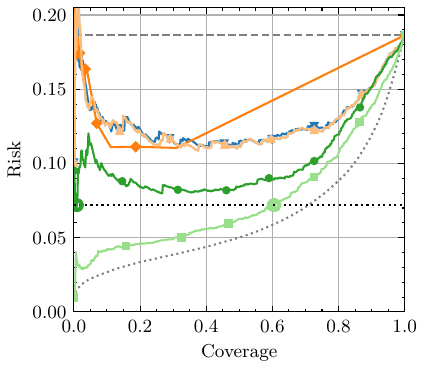}
        \caption{Polyp}
    \end{subfigure}

    \caption{Risk-coverage curves on OOD datasets using non-tunable confidence estimators. Circle marker indicates the point at which the selective prediction model matches the full-coverage ID risk.}
    \label{fig:result_rc_curves_OOD}
\end{figure}
shows the same evaluation on OOD data, where we additionally indicate the (full-coverage) risk on the corresponding ID data; this is relevant since one possible use case for selective classification on OOD data is to allow a model to achieve the same performance obtained on ID data, albeit at a smaller coverage. When achievable, the operating point where the selective and full-coverage ID risks coincide is denoted by a circle. Again, the SDC consistently achieves the lowest risk across all coverage levels. Notably, \new{although it is not the only method that reaches, at some non-zero coverage, a selective risk equivalent to the full-coverage ID risk, it remains the only one demonstrating this behavior consistently across datasets. In contrast, PLA attains such a point only for the Polyp OOD scenario, at a marginal coverage of approximately 0.79\%.}




\subsection{Experiments with tunable confidence estimators}
\label{ssec:experiments-tunable}


We now compare the SDC against tunable confidence estimators.
Following \citet{cattelan_how_2023}, we simulate a realistic scenario where we have a single dataset that must serve as both a tuning set (for training confidence score functions) and a test set (for evaluation). In our experiments, we randomly partition the reserved test data into a tuning set and a new test set. For the tunable baseline, we use AEF, which predicts segmentation performance by leveraging automatically extracted features. Consistent with \citet{jungo_analyzing_2020}, we follow a similar approach but increase the number of estimators in the random forest from 10 to 100 to improve predictive performance. Additionally, we propose a combination of AEF with SDC as an additional feature, referred to as AEF+SDC. \new{Finally, we also include PLA* as a tunable confidence estimator. In this case, the tuning set is used to select the optimal patch size from a grid ranging from 1 to 200, with finer increments for smaller patches, based on the configuration that yields the lowest AURC. The selected patch size is then applied to compute the results on the test set.}

To quantify the impact of tuning set size, we vary the proportion of data allocated to training the confidence score functions, from as few as 2 images up to half of the dataset, and measure the resulting test AURC. For each split, we repeat the random partitioning 50 times and report mean and one standard deviation (as a solid curve and a shaded region, respectively). We also include the Random and Oracle references, as the upper and lower dashed lines, respectively.

Figure~\ref{fig:result_tuning} 
\begin{figure}
    \centering
    \begin{subfigure}[b]{0.4\textwidth} 
        \vskip 0pt
        \includegraphics[width=\textwidth]{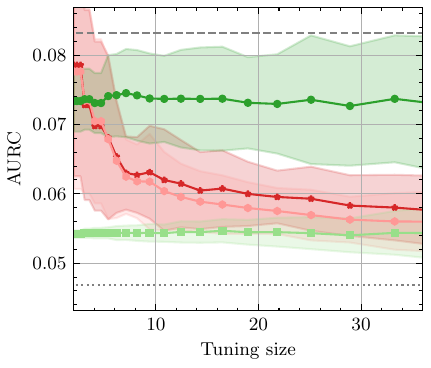}
        \caption{Brain Tumor}
    \end{subfigure}
    \hspace{0.05\textwidth} 
    \begin{subfigure}[b]{0.4\textwidth} 
        \vskip 0pt
        \includegraphics[width=\textwidth]{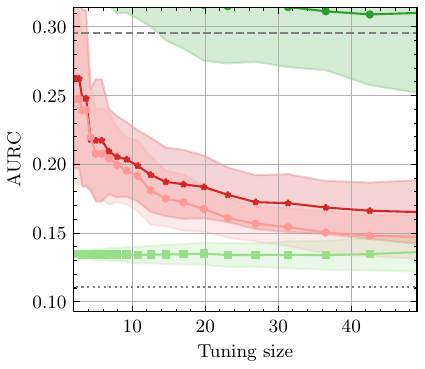}
        \caption{Breast Cancer}
    \end{subfigure}

    \vspace{0.2cm} 
    \begin{subfigure}[b]{0.4\textwidth}
        \vskip 0pt
        \includegraphics[width=\textwidth]{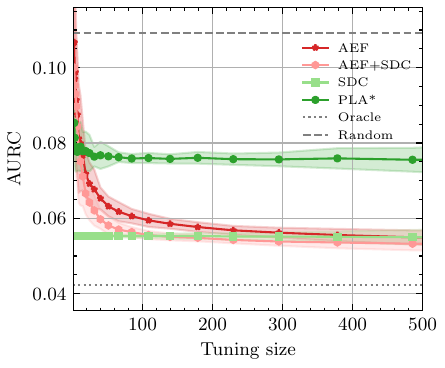}
        \caption{Skin Cancer}
    \end{subfigure}
    \hspace{0.05\textwidth} 
    \begin{subfigure}[b]{0.4\textwidth}
        \vskip 0pt
        \includegraphics[width=\textwidth]{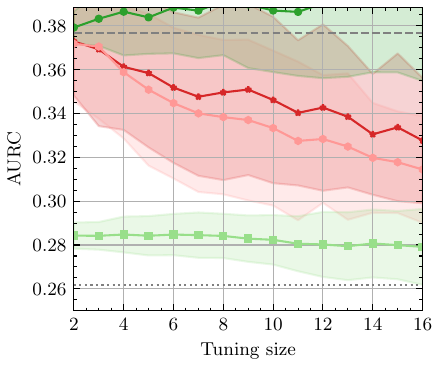}
        \caption{MSWML}
    \end{subfigure}

    \vspace{0.2cm} 
    \begin{subfigure}[b]{0.4\textwidth}
        \vskip 0pt
        \includegraphics[width=\textwidth]{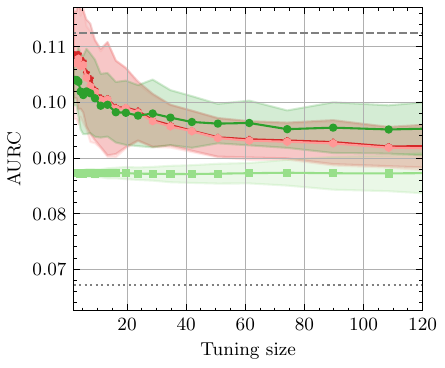}
        \caption{Optic Cup}
    \end{subfigure}
    \hspace{0.05\textwidth} 
    \begin{subfigure}[b]{0.4\textwidth}
        \vskip 0pt
        \includegraphics[width=\textwidth]{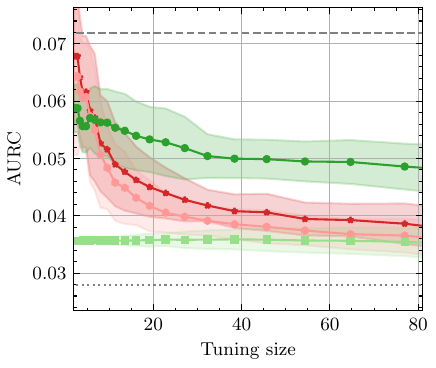}
        \caption{Polyp}
    \end{subfigure}

    \caption{AURC as a function of the tuning set size on ID datasets.} 
    \label{fig:result_tuning}
\end{figure}
shows the results on ID data. SDC consistently outperforms tunable baselines, even with very large tuning sets. As the tuning size grows, tunable estimators may approach SDC's performance; however, the reliance on more data can be a practical limitation. In contrast, SDC achieves results close to the oracle bound across a wide range of conditions without requiring additional data.

\new{Figure~\ref{fig:result_tuning_aurc_OOD} shows the same evaluation on OOD data.} We can see that SDC outperforms tunable methods even with small tuning sets and, while the performance of tunable methods converges toward that of SDC as the tuning set size increases, SDC remains competitive, particularly in the MSWML task where its performance nearly reaches the oracle bound.

\new{While the AURC provides a convenient summary of the selective prediction performance across all coverage levels, it may be difficult to interpret. To give a more palpable measure of performance, we select a meaningful operating point on the RC curve, namely, the selective risk equal to the full-coverage risk on ID data, and report the corresponding maximum coverage on OOD data that achieves that risk. As discussed in the previous section, this measure answers the practical question of how much of the OOD data can be safely accepted while maintaining a level of performance that is at least as good as the baseline performance achieved on ID data without rejection, with selective prediction acting to mitigate performance degradation by rejecting uncertain cases. The corresponding results are shown in Figure~\ref{fig:result_tunable_curves_OOD}, where we again vary the tuning set size and report mean and one standard deviation as described above. As can be seen, the SDC outperforms all other confidence estimators on all datasets, without any need for tuning data.}



%
\begin{figure}
    \centering
    \begin{subfigure}[t]{0.4\textwidth} 
        \vskip 0pt
        \includegraphics[width=\textwidth]{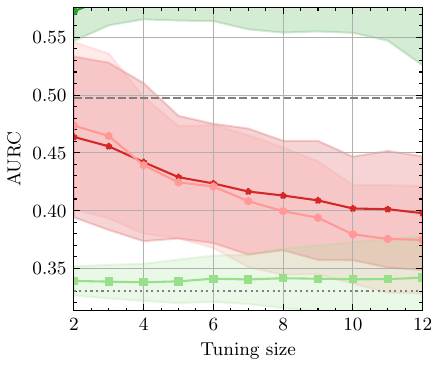}
        \caption{MSWML}
    \end{subfigure}
    \hspace{0.05\textwidth} 
    \begin{subfigure}[t]{0.4\textwidth} 
        \vskip 0pt
        \includegraphics[width=\textwidth]{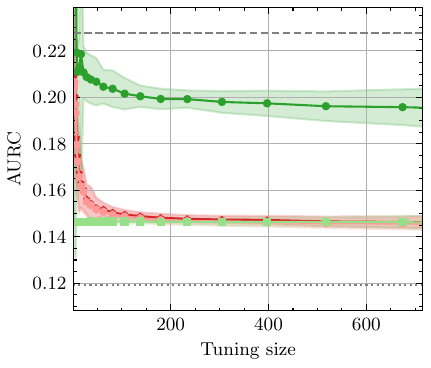}
        \caption{Optic Cup}
    \end{subfigure}  

    \vspace{0.2cm} 
    \begin{subfigure}[t]{0.4\textwidth}
        \includegraphics[width=\textwidth]{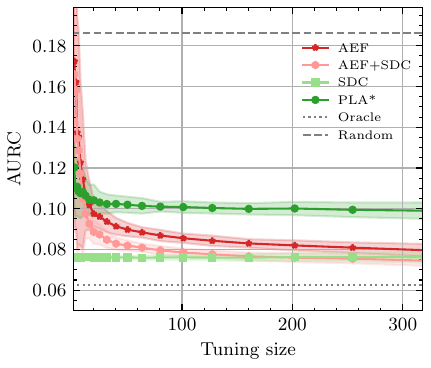}
        \caption{Polyp}
    \end{subfigure}

    \caption{\new{AURC as a function of the tuning set size on OOD datasets.}}
    \label{fig:result_tuning_aurc_OOD}
\end{figure}
\begin{figure}
    \centering

    \begin{subfigure}[b]{0.4\textwidth} 
        \vskip 0pt
        \includegraphics[width=\textwidth]{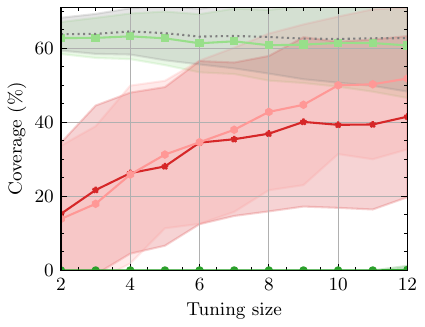}
        \caption{MSWML}
    \end{subfigure}
    \hspace{0.05\textwidth} 
    \begin{subfigure}[b]{0.4\textwidth} 
        \vskip 0pt
        \includegraphics[width=\textwidth]{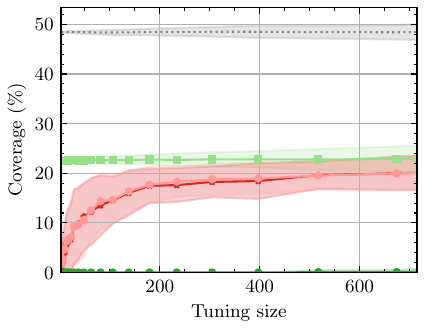}
        \caption{Optic Cup}
    \end{subfigure}  

    \vspace{0.2cm} 
    \begin{subfigure}[b]{0.4\textwidth}
        \includegraphics[width=\textwidth]{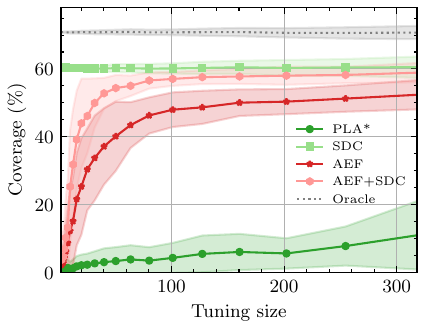}
        \caption{Polyp}
    \end{subfigure}

    \caption{Maximum OOD coverage achieved by tunable confidence estimators and SDC at equivalent ID risk.}
    \label{fig:result_tunable_curves_OOD}
\end{figure}
%



\subsection{Summary}

Table~\ref{tab:aurcs_id_ood} presents the AURC results for all confidence estimators, with performance reported for both ID and OOD data. For the tunable estimators, the mean and standard deviation over 50 random splits are provided for each tuning set proportion. As observed in the table, SDC achieves the best performance across all datasets and scenarios (both ID and OOD), with the only exception being the Skin Cancer dataset, where AEF+SDC with a 50\% tuning proportion slightly outperforms SDC in terms of AURC. However, when considering the standard deviation, AEF+SDC yields results comparable to SDC.

\begin{table}
\caption{AURCs of non-tunable confidence estimators for ID and OOD data. Lower values are better. For the tunable confidence estimators AEF and AEF+SDC, tuning proportions of 10\% and 50\% were used. Bold numbers indicate lowest values (on average).}
\setlength{\tabcolsep}{5pt}
\renewcommand{\arraystretch}{1.15}
\centering
\begin{tabular}{lccccccccc}
\toprule
 & \new{\makecell{Brain \\ Tumor}} & \makecell{Breast \\ Cancer} & \makecell{Skin \\ Cancer} & \multicolumn{2}{c}{MSWML} & \multicolumn{2}{c}{Optic Cup} & \multicolumn{2}{c}{Polyp} \\
 & ID & ID & ID & ID & OOD & ID & OOD & ID & OOD \\
\midrule
aMSP   & 0.088 & 0.341 & 0.085 & 0.414 & 0.600 & 0.116 & 0.270 & 0.068 & 0.129 \\
aNE    & 0.089 & 0.343 & 0.083 & 0.418 & 0.595 & 0.117 & 0.268 & 0.068 & 0.128 \\
MMMC   & 0.081 & 0.254 & 0.074 & 0.394 & 0.544 & 0.113 & 0.605 & 0.063 & 0.141 \\
\new{PLA}   & 0.070 & 0.329 & 0.075 & 0.408 & 0.603 & 0.101 & 0.220 & 0.051 & 0.103 \\
TLA   & 0.089 & 0.342 & 0.084 & 0.412 & 0.612 & 0.116 & 0.268 & 0.068 & 0.128 \\
SDC    & \textbf{0.054} & \textbf{0.135} & 0.055 & \textbf{0.282} & \textbf{0.338} & \textbf{0.087} & \textbf{0.146} & \textbf{0.036} & 0.076 \\
\new{PLA* (10\%)} & \meanstd{0.074}{0.006} & \meanstd{0.331}{0.022} & \meanstd{0.076}{0.001} & \meanstd{0.383}{0.012} & \meanstd{0.568}{0.027} & \meanstd{0.098}{0.006} & \meanstd{0.200}{0.005} & \meanstd{0.054}{0.006} & \meanstd{0.101}{0.004} \\
\new{PLA* (50\%)} & \meanstd{0.073}{0.010} & \meanstd{0.310}{0.058} & \meanstd{0.076}{0.003} & \meanstd{0.395}{0.041} & \meanstd{0.582}{0.059} & \meanstd{0.095}{0.005} & \textbf{\meanstd{0.196}{0.008}} & \meanstd{0.048}{0.004} & \meanstd{0.099}{0.004} \\
AEF (10\%) & \meanstd{0.063}{0.005} & \meanstd{0.202}{0.027} & \meanstd{0.060}{0.002} & \meanstd{0.365}{0.033} & \meanstd{0.465}{0.060} & \meanstd{0.098}{0.005} & \meanstd{0.149}{0.002} & \meanstd{0.046}{0.006} & \meanstd{0.088}{0.003} \\
AEF (50\%) & \meanstd{0.058}{0.005} & \meanstd{0.165}{0.023} & \meanstd{0.055}{0.002} & \meanstd{0.330}{0.038} & \meanstd{0.404}{0.051} & \meanstd{0.092}{0.004} & \textbf{\meanstd{0.146}{0.003}} & \meanstd{0.038}{0.004} & \meanstd{0.080}{0.003} \\
AEF+SDC (10\%) & \meanstd{0.062}{0.005} & \meanstd{0.194}{0.025} & \meanstd{0.056}{0.001} & \meanstd{0.368}{0.033} & \meanstd{0.469}{0.061} & \meanstd{0.098}{0.005} & \meanstd{0.148}{0.002} & \meanstd{0.043}{0.005} & \meanstd{0.081}{0.002} \\
AEF+SDC (50\%) & \meanstd{0.056}{0.004} & \meanstd{0.147}{0.016} & \textbf{\meanstd{0.053}{0.002}} & \meanstd{0.316}{0.030} & \meanstd{0.374}{0.049} & \meanstd{0.092}{0.004} & \textbf{\meanstd{0.146}{0.003}} & \textbf{\meanstd{0.036}{0.003}} & \textbf{\meanstd{0.075}{0.003}} \\
\midrule
Oracle & 0.047 & 0.111 & 0.042 & 0.262 & 0.330 & 0.067 & 0.119 & 0.028 & 0.062 \\
\bottomrule
\end{tabular}
\label{tab:aurcs_id_ood}
\end{table}

Finally, Table~\ref{tab:coverage} further illustrates the potential of selective prediction on OOD data by reporting the maximum OOD coverage that achieves a selective risk equal to full-coverage risk on ID data. Similar to Table~\ref{tab:aurcs_id_ood}, the results for tunable confidence estimators are provided as the mean and standard deviation over 50 random splits for each tuning set proportion. The table shows that SDC outperforms the other estimators by a wide margin. 

\begin{table}
\caption{Maximum coverage of the OOD dataset using the designated confidence estimator that reproduces the ID risk at 100\% coverage (higher is better). Blank values (-) indicate that the corresponding ID risk could not be achieved for any coverage levels. For the tunable confidence estimators AEF and AEF+SDC, tuning proportions of 10\% and 50\% were used. Bold numbers indicate highest values (on average).}
\centering
\renewcommand{\arraystretch}{1.15}
\begin{tabular}{@{} p{4cm} *{3}{>{\centering\arraybackslash}p{2cm}} @{}}
\toprule
 & MSWML & Optic Cup & Polyp \\
\midrule
aMSP           & -         & -         & -         \\
aNE            & -         & -         & -         \\
MMMC           & -         & -         & -         \\
\new{PLA}            & -         & -         & 0.79\%    \\
TLA            & -         & -         & -         \\
SDC            & \textbf{64.00\%}   & \textbf{22.45\%}   & \textbf{60.53\%}   \\
\new{PLA* (10\%)}     & - & -  & \meanstd{3.86\%}{4.19\%} \\
\new{PLA* (50\%)}     & \meanstd{0.41\%}{2.42\%} & \meanstd{0.17\%}{0.61\%}  & \meanstd{10.97\%}{10.07\%} \\
AEF (10\%)     & \meanstd{13.30\%}{16.63\%} & \meanstd{16.22\%}{4.38\%}  & \meanstd{43.43\%}{6.79\%} \\
AEF (50\%)     & \meanstd{37.81\%}{23.45\%} & \meanstd{20.09\%}{3.46\%}  & \meanstd{52.39\%}{4.32\%} \\
AEF+SDC (10\%) & \meanstd{10.96\%}{17.24\%} & \meanstd{16.52\%}{4.19\%}  & \meanstd{55.02\%}{3.02\%} \\
AEF+SDC (50\%) & \meanstd{50.68\%}{20.75\%} & \meanstd{20.04\%}{3.41\%}  & \meanstd{58.90\%}{2.85\%} \\
\midrule
Oracle         & 64.00\%   & 48.50\%   & 70.90\%   \\
\bottomrule
\end{tabular}
\label{tab:coverage}
\end{table}


\newpage
\clearpage

\section{Conclusions}


In this work, we introduced the Soft Dice Confidence (SDC), a novel confidence score function tailored for semantic segmentation tasks.
Our theoretical results show that, under certain conditions, the SDC, despite being computable in linear time, can closely approximate the ideal confidence score function (IDC), which is intractable for real-world data. Empirical evaluations across diverse medical imaging tasks, including in-distribution and out-of-distribution data, corroborate the theoretical results, with SDC outperforming all other confidence estimators proposed in the literature.

The SDC differs from the ideal confidence estimator in three ways: (i) it is an approximation to the expression of the ideal confidence estimator computed solely from the marginal posteriors, which in turn (ii) is only equal to the true ideal confidence estimator based on the full posteriors when conditional independence holds; and (iii) in practice, we only have access to an estimate of the marginal posteriors. The first difference is tightly bounded by our theoretical results; in particular, Corollary~\ref{cor:error-bound-eps} gives a practical way to bound the relative error, which we have extensively confirmed in our experimental results to be extremely small in all scenarios considered. As for conditional independence, while it is unlikely to be satisfied exactly in practice, we should note that it is a typical modeling assumption in semantic segmentation, due to the sheer impracticality of modeling the full posterior in any meaningfully-sized segmentation problem. \newnew{Nevertheless, it remains as a limitation of our approach.} The third difference, on the other hand, appears to be the most promising as a direction for future improvements. As we show in our experiments with simulated data, estimation errors on the marginal posteriors can significantly degrade the performance of the SDC (as well as of the IDC when computed with the same estimated marginals, to which the SDC remains tightly bounded). 

Naturally, one can always approach this estimation problem explicitly, by pursuing a better probabilistic model underlying the segmentation model. Alternatively, one can approach it implicitly, by tuning hyperparameters of a confidence estimator on additional data not used for training the model. We have observed such tuning to be beneficial to the tunable confidence estimators evaluated in this work; however, none was able to outperform the SDC, except in a single case where it was surpassed slightly by a hybrid of AEF~\citep{jungo_analyzing_2020} and SDC itself, after consuming an expressive amount of tuning data. In contrast, the SDC does not require tuning and is extremely efficient to compute. Its simplicity and strong empirical performance make it an appealing choice, particularly in low-resource settings where tuning data is scarce.

It is noteworthy that the SDC achieved comparable performance to the Oracle bound in tasks like MSWML and Breast Cancer segmentation, implying that it is very close to the ideal confidence estimators in such cases. On the other hand, it presented a significant gap in tasks such as Optic Cup segmentation, showing that there is still room for improvement. One suggestion for future work is to develop a tunable confidence estimator that generalizes the SDC, as it would already start from a theoretically grounded foundation. 

\newpage

\appendix

\section{Proof of Theorem \ref{theo:main}}\label{app:theorem}

We start with a useful technical lemma.

\begin{lemma}\label{lemma:prob-function-max}
    Let $f:[0,1]^n \to \R^+$ be a function of the form
    \begin{equation}
    f(x_1,\ldots,x_n) = \sum_{\by \in \mathcal{Y}} \left( \prod_{i=1}^{n} x_i^{y_i}(1-x_i)^{1-y_i} \right)  \frac{1}{c + \sum_{j=1}^{n} y_j},
    \end{equation} where $c>0$.
    Then, for any $s\in [0, n]$,
    \begin{equation}\label{eq:lemma1-max-problem}
\left( \mu,\ldots,\mu \right) = \argmax_{\substack{x_1,\ldots,x_n \in [0,1]: \\ x_1+\cdots+x_n = s}}\, f\left( x_1,\ldots,x_n \right)
    \end{equation}
    where $\mu = s/n$.
\end{lemma}

\begin{proof}
    The case $n=1$ is trivial as $x=\mu=s$ is the only solution to \eqref{eq:lemma1-max-problem}.
    For functions with $n=2$ variables, simple algebraic manipulation shows that
    \begin{align*}
        f(x_1,x_2) &= \sum_{y_1,y_2 \in \left\{ 0,1 \right\} } \left( x_1^{y_1}(1-x_1)^{1-y_1} x_2^{y_2}(1-x_2)^{1-y_2} \right)  \frac{1}{c + y_1 + y_2}  \\
	&= \frac{1}{c} - \frac{1}{c(c+1)}(x_1+x_2) + \frac{2}{c(c+1)(c+2)}x_1x_2
    .\end{align*}
    Given $x_1+x_2=s$, the above can be written as \[
    f(x_1,x_2) = \frac{1}{c} - \frac{1}{c(c+1)}s + \frac{2s}{c(c+1)(c+2)}x_2 - \frac{2}{c(c+1)(c+2)}x_2^2
    ,\] which has a maximum at $x_2 = s / 2 = \mu$, implying that $x_1 = \mu$ as well.

    Now assume that \eqref{eq:lemma1-max-problem} is true for functions $f$ with up to $n-1$ variables (inductive hypothesis).
    Then, it holds for
    \begin{equation}\label{eq:definition-parameterized-function}
	f(x_1,\ldots,x_{n-1};b) = \sum_{y_1,\ldots,y_{n-1}} \left( \prod_{i=1}^{n-1} x_{i}^{y_{i}}(1-x_{i})^{1-y_{i}} \right)  \frac{1}{b + y_1+\cdots +y_{n-1}} 
    ,\end{equation}
    as long as $b>0$.
    Furthermore, $f(x_1,\ldots,x_n)$ can be rewritten as
    \begin{equation*}
	f(x_1,\ldots,x_n) = \underbrace{(1-x_n)f(x_1,\ldots,x_{n-1};c)}_{y_n=0} + \underbrace{x_n f(x_1,\ldots,x_{n-1};c+1)}_{y_n=1}
    .\end{equation*}

    We want to show that if $\left( x_1^{*},\ldots,x_{n}^{*} \right) \in \left[ 0,1 \right]^{n} $ is a point that maximizes $f$ subject to $x_1^{*}+\cdots +x_{n}^{*}=s$, then $x_1^{*}=\cdots=x_{n}^{*}$.
    We have that 
    \begin{align}
	f(x_1^{*},\ldots,x_{n}^{*}) &= \max_{\substack{x_1,\ldots,x_{n-1} \in \left[ 0,1 \right]: \\ x_1+\cdots + x_{n-1} + x_{n}^{*}=s }}\, f(x_1,\ldots,x_{n-1},x_{n}^{*}) \nonumber \\
	&= \max_{\substack{x_1,\ldots,x_{n-1} \in \left[ 0,1 \right]: \\ x_1+\cdots + x_{n-1} =s-x_{n}^{*} }}\, (1-x_n^{*})f(x_1,\ldots,x_{n-1};c) + x_n^{*} f(x_1,\ldots,x_{n-1};c+1) \nonumber \\
    \begin{split}
	    &\le (1-x_{n}^{*}) \max_{\substack{x_1,\ldots,x_{n-1} \in \left[ 0,1 \right]:  \\ x_1+\cdots +x_{n-1}=s-x_{n}^{*}}}\, f(x_1,\ldots,x_{n-1};c) \\
	    &\quad\; + x_{n}^{*} \max_{\substack{x_1,\ldots,x_{n-1} \in \left[ 0,1 \right]: \\ x_1+\cdots +x_{n-1}=s-x_{n}^{*}}}\, f(x_1,\ldots,x_{n-1};c+1)
	\end{split} \nonumber \\
	&= (1-x_n^{*})f(\mu',\ldots,\mu';c) + x_n^{*} f(\mu',\ldots,\mu';c+1) \label{eq:parameterized-maximization} \\    
	&= f(\mu',\ldots,\mu',x_n^*) \nonumber 
    \end{align}
    where $\mu' = (s-x_n^*)/(n-1)$ and \eqref{eq:parameterized-maximization} follows from the inductive hypothesis. Thus, $x_1^{*}=\cdots=x_{n-1}^{*}$. 
    
    However, the choice for leaving $x_{n}$ out of $f(x_1,\ldots,x_{n-1};b)$ in \eqref{eq:definition-parameterized-function} was arbitrary.
    For example, we can just as well rewrite $f(x_1,\ldots,x_n)$ as \[
	f(x_1,\ldots,x_n) = \underbrace{(1-x_1)f(x_2,\ldots,x_{n};c)}_{y_1=0} + \underbrace{x_1 f(x_2,\ldots,x_{n};c+1)}_{y_1=1}
    ,\] which shows that $x_2^{*}=\cdots=x_{n}^{*}$ must hold for any optimum $(x_1^{*},\ldots,x_{n}^{*})$.
    Therefore, $x_1^{*}=\cdots=x_{n}^{*}=s / n=\mu$.
\end{proof}

\begin{corollary}\label{cor:prob-function-max}
    Let $f$ and $c$ be as in Lemma~\ref{lemma:prob-function-max}.
    If $s>0$, then
    \begin{equation}\label{eq:limit-prob-function-poisson}
    f(\mu,\ldots,\mu) \le \sum_{i=0}^{\infty} \frac{s^{i} e^{-s}}{i!} \frac{1}{c+i}
    ,\end{equation}
    with equality achieved only at the limit as $n\to \infty$.
\end{corollary}

\begin{proof}
    First, we note that $f(x_1,\ldots,x_{n},0) = f(x_1,\ldots,x_n)$.
    Therefore, by Lemma~\ref{lemma:prob-function-max}, for all $m\ge n$,
    \begin{align*}
        f(\mu,\ldots,\mu) &=  \max_{\substack{x_1,\ldots,x_n \in [0,1]: \\ x_1+\cdots+x_n = s}}\, f\left( x_1,\ldots,x_n \right) \\
	&= \max_{\substack{x_1,\ldots,x_n \in [0,1]: \\ x_1+\cdots+x_n = s}}\, f( x_1,\ldots,x_n,\underbrace{0,\ldots,0}_{m-n} ) \\
        &\le \max_{\substack{x_1,\ldots,x_m \in [0,1]: \\ x_1+\cdots+x_m = s}}\, f\left( x_1,\ldots,x_n,x_{n+1},\ldots,x_m \right) \\
	&= f(s / m,\ldots,s / m) \\
	&= \mathbb{E}_{w\sim \text{Binomial}(m, s / m)} \left[ \frac{1}{c + w} \right]
    \end{align*}
    with equality if and only if $m=n$.
    By taking the limit $m\to \infty$, $w$ becomes a Poisson random variable with parameter $s$, such that the expected value is as in \eqref{eq:limit-prob-function-poisson}.
\end{proof}

Before proceeding to the next results, we introduce a simplified notation. Note that, due to symmetry, $D(\by, \byh)$, $\IDC(\bp, \byh)$ and $\SDC(\bp, \byh)$ are all invariant to the same permutation of both of their arguments, i.e., if $\pi: \RR^n \to \RR^n$ is a permutation, then, e.g., $\IDC(\pi(\bp), \pi(\byh)) = \IDC(\bp, \byh)$. Thus, without loss of generality, we only need to consider the case $\hat{y}_1=\cdots=\hat{y}_k=1$, $\hat{y}_{k+1}=\cdots=\hat{y}_n=0$, where $k \in \{0,\ldots,n\}$.
This allows us to write $\IDC(\bp, \byh) = d_k(\bp)$ and $\SDC(\bp, \byh) = \sdc_k(\bp)$, where
\begin{align}
d_k(\bp) &\triangleq 
    \sum_{\by\in \mathcal{Y}} \left(  \prod_{i=1}^n p_{i}^{y_{i}} (1-p_{i})^{1-y_{i}}  \right)   
    \frac{2\sum_{j = 1}^k y_j }{k + \sum_{j=1}^n y_j} \\
\sdc_k(\bp) &\triangleq 
    \frac{2\sum_{j =1}^k p_j }{k + \sum_{j=1}^n p_j}
\end{align}
if $k>0$ and $d_0(\bp) = \sdc_0(\bp) = 0$.

Under this new notation, for $k>0$, we can express
\begin{equation}
d_k(\bp) = \mathbb{E}_{\substack{y_i \sim \text{Bernoulli}(p_i), \\ i=1,\ldots,n}}\left[ \frac{2 w_1(\by)}{k + w_1(\by) + w_0(\by)} \right] 
\end{equation}
where $w_1(\by) = \sum_{j = 1}^k y_j$ and $w_0(\by) = \sum_{j=k+1}^n y_j$.
Intuitively, $w_1(\by)$ and $w_0(\by)$ represent, respectively, the number of true-positives and false-negatives of $\byh$ with respect to $\by$.
Because each $y_i\sim \text{Bernoulli}(p_i)$, we can reframe the above expression in terms of $w_0$ and $w_1$ as variables following Poisson Binomial distributions, such that 
\begin{equation}
\label{eq:idc-expectation}
d_{k}(\bp) = \mathbb{E}_{w_1} \left[ \mathbb{E}_{w_0 } \left[ \frac{2 w_1}{k + w_1 + w_0} \right]  \right] 
\end{equation}
where $w_1 \sim \text{PoissonBinomial}(p_{1},\ldots,p_k)$ and $w_0\sim \text{PoissonBinomial}(p_{k+1},\ldots,p_n)$. This formulation will be used in the proofs of the following results.

\begin{lemma}\label{lemma:dice-lower-bound}
    For any $\bp \in \left[ 0,1 \right]^{n}$, $n\ge 1$, and $1\le k\le n$, 
    \begin{align}
d_k(\bp) & \ge 2k\mu \sum_{i=0}^{k-1} \begin{pmatrix} k-1 \\ i \end{pmatrix} \mu^{i}\left( 1 -\mu \right)^{k-1-i} \frac{1}{k+1+i+ \sum_{j=k+1}^{n} p_j} \label{eq:dice-lower-bound-1} \\
		 & \ge \frac{2k\mu}{k+1+(k-1)\mu + \sum_{j=k+1}^{n} p_i}\label{eq:dice-lower-bound-2} 
    \end{align}
    where $\mu = \frac{1}{k} \sum_{i=1}^{k} p_i$.
    Moreover, \eqref{eq:dice-lower-bound-1} is an equality if and only if $\bp$ is such that $p_1=\cdots=p_k=\mu$ and $p_{k+1},\ldots,p_n\in \left\{ 0,1 \right\} $.
    Furthermore, \eqref{eq:dice-lower-bound-2} is an equality if and only if \eqref{eq:dice-lower-bound-1} is an equality and $k=1$ or $\mu\in \left\{ 0,1 \right\}$.
\end{lemma}

\begin{proof}
Starting from \eqref{eq:idc-expectation}, since $w_1 / \left( k+w_0 + w_1 \right) $ is convex in $w_0$, it follows by Jensen's inequality that
    \begin{equation}\label{eq:dk-as-expectation-of-w1}
    \begin{split}
	\frac{1}{2} d_k(\bp) &\ge \mathbb{E}_{w_1} \left[ \frac{w_1}{k + w_1 + \mathbb{E}_{w_0}\left[ w_0 \right] } \right]  \\
			     &= \mathbb{E}_{w_1} \left[ \frac{w_1}{k + w_1 + \lambda } \right]
    ,\end{split}
    \end{equation}
    where $\lambda = \sum_{i=k+1}^{n} p_i$, with equality if and only if $p_{k+1},\ldots,p_n \in \left\{ 0,1 \right\} $.

    Now, let $\mu = \frac{1}{k} \sum_{i=1}^{k} p_i$.
    By Lemma~\ref{lemma:prob-function-max}, given any $c>0$, 
    \begin{align*}
        \mathbb{E}_{w_1} \left[ \frac{1}{c+w_1} \right] &= \mathbb{E}_{\substack{y_i \sim \text{Bernoulli}(p_i), \\ i=1,\ldots,k}} \left[ \frac{1}{c + \sum_{j=1}^{k} y_j} \right]  \\
	&= \sum_{y_1,\ldots,y_k\in \left\{ 0,1 \right\}} \left( \prod_{i=1}^{k} p_i^{y_i}(1-p_i)^{1-y_i}  \right) \frac{1}{c + \sum_{j=1}^{k} y_j} \\
	&\le \sum_{y_1,\ldots,y_k\in \left\{ 0,1 \right\}} \left( \prod_{i=1}^{k} \mu^{y_i}(1-\mu)^{1-y_i}  \right) \frac{1}{c + \sum_{j=1}^{k} y_j} \\
        &= \mathbb{E}_{\substack{y_i \sim \text{Bernoulli}(\mu), \\ i=1,\ldots,k}} \left[ \frac{1}{c + \sum_{j=1}^{k} y_j} \right] = \mathbb{E}_{w_1'} \left[ \frac{1}{c+w_1'} \right]
    \end{align*}
    where $w_1'\sim \text{Binomial}(k,\mu)$.
    Therefore,
    \begin{equation}\label{eq:w1-prime-inequality}
    \begin{split}
        \mathbb{E}_{w_1} \left[ \frac{w_1}{k+\lambda+w_1} \right]
	&= 1 - \mathbb{E}_{w_1} \left[ \frac{k+\lambda}{k+\lambda+w_1} \right]  \\
	&\ge 1 - \mathbb{E}_{w_1'} \left[ \frac{k+\lambda}{k+\lambda+w_1'} \right] 
	= \mathbb{E}_{w_1'}\left[ \frac{w_1'}{k+\lambda+w_1'} \right].
    \end{split}
    \end{equation}

    Together, \eqref{eq:dk-as-expectation-of-w1} and \eqref{eq:w1-prime-inequality} yield
    \begin{align*}
        d_k(\bp) &\ge 2 \mathbb{E}_{w_1'}\left[ \frac{w_1'}{k+\lambda+w_1'} \right]  \\
	&= 2\sum_{j=0}^{k} \begin{pmatrix} k \\ j \end{pmatrix} \mu^{j}\left( 1-\mu \right)^{k-j} \frac{j}{k+\lambda+j}  \\
	&= 2\sum_{i=0}^{k-1} \frac{k}{i+1} \begin{pmatrix} k-1\\i \end{pmatrix} \mu^{i+1}(1-\mu)^{k-1-i} \frac{i+1}{k+\lambda+i+1} \\
	&= 2k \mu \sum_{i=0}^{k-1} \begin{pmatrix} k-1\\i \end{pmatrix} \mu^{i}(1-\mu)^{k-1-i} \frac{1}{k+\lambda+i+1}
    \end{align*}
    which is precisely \eqref{eq:dice-lower-bound-1}.
    Note that the above is an equality if and only if \eqref{eq:dk-as-expectation-of-w1} and \eqref{eq:w1-prime-inequality} are equalities, which happens if and only if $p_1=\cdots=p_k=\mu$ and $p_{k+1},\ldots,p_n\in \left\{ 0,1 \right\}$.
    
    Finally, to show that \eqref{eq:dice-lower-bound-2} holds, let $w_1''\sim \text{Binomial}(k-1,\mu)$. Then, by Jensen's inequality,
    \begin{align*}
	2k \mu \sum_{i=0}^{k-1} \begin{pmatrix} k-1\\i \end{pmatrix} \mu^{i}(1-\mu)^{k-1-i} \frac{1}{k+\lambda+i+1}
	&= 2k\mu \mathbb{E}_{w_1''} \left[ \frac{1}{k+\lambda+w_1''+1} \right] \\
	&\ge \frac{1}{k+\lambda +\mathbb{E}[w_1''] +1}  \\
	&= \frac{1}{k+\lambda+(k-1)\mu + 1}
    \end{align*}
    with equality if and only if $k=1$ or $\mu \in \left\{ 0,1 \right\}$.
\end{proof}

\begin{lemma}\label{lemma:dice-upper-bound}
    For any $\bp \in \left[ 0,1 \right]^{n}$, $n\ge 1$, and $1\le k\le n$, 
    \begin{equation}
        d_k(\bp) \le \sum_{i=0}^{\infty} \frac{\lambda^{i} e^{-\lambda}}{i!} \frac{2k\mu}{k+k\mu+i}
    \end{equation}
    where $\mu=\frac{1}{k}\sum_{j=1}^{k} p_j$ and $\lambda = \sum_{j=k+1}^{n} p_j$, with equality at the limit $n-k\to \infty$, provided that $p_1,\ldots,p_k \in \left\{ 0,1 \right\}$ and $p_{k+1}=\cdots=p_n=\lambda / (n-k)$.
\end{lemma}
\begin{proof}
Starting from \eqref{eq:idc-expectation}, by Jensen's inequality we have
    \begin{equation}\label{eq:dice-upper-bound-ineq-1}
	d_k(\bp) = \mathbb{E}_{w_0} \left[ \mathbb{E}_{w_1} \left[ \frac{2w_1}{k+w_1+w_0} \right]  \right] \le \mathbb{E}_{w_0} \left[ \frac{2k\mu}{k+k\mu+w_0} \right].
    \end{equation}
    By Corollary~\ref{cor:prob-function-max},
    \begin{equation}
    \begin{split}\label{eq:dice-upper-bound-ineq-2}
        \mathbb{E}_{w_0} \left[ \frac{2k\mu}{k+k\mu+w_0} \right] &= \sum_{y_{k+1},\ldots,y_{n}\in \left\{ 0,1 \right\}} \left( \prod_{i=k+1}^{n} p_i^{y_i}(1-p_i)^{1-y_i}  \right) \frac{2k\mu}{k+ k\mu + \sum_{j=k+1}^{n} y_j} \\
	&\le \sum_{i=0}^{\infty} \frac{\lambda^{i}e^{-\lambda}}{i!} \frac{2k\mu}{k+k\mu+i}
    .\end{split}
    \end{equation}

    Note that \eqref{eq:dice-upper-bound-ineq-1} is an equality if and only if $p_1,\ldots,p_k \in \left\{ 0,1 \right\} $, while \eqref{eq:dice-upper-bound-ineq-2} is an equality if and only if $p_{k+1}=\cdots=p_n=\lambda / (n-k)$ at the limit $n-k \to \infty$.
\end{proof}

We can now proceed to the proof of our main result.
\bigskip

\begin{proof}\textit{\hspace{-0.5ex}of Theorem~\hyperref[theo:main]{\ref{theo:main}}} 

In order to prove statement \eqref{eq:theo-null-statement}, first note that $k=0$ (i.e., $\byh = \bzero$) immediately implies $d_k(\bp) = \sdc_k(\bp) = s = 0$. Thus, assume $k>0$. In this case, it is easy to see that $\sdc_k(\bp) = 0 \iff s = 0$. Moreover, because of \eqref{eq:idc-expectation}, we can readily see that $d_k(\bp) = 0 \iff w_1 = 0 \iff p_1=\cdots=p_k =0 \iff s = 0$.

The lower bound in \eqref{eq:sdc-bounds-inequalities} follows from Lemma~\ref{lemma:dice-lower-bound}, since
\begin{align*}
\frac{d_k(\bp)}{\widetilde{d}_k(\bp)}  
&\ge \frac{1}{\widetilde{d}_k(\bp)} \frac{2k\mu}{k+1+(k-1)\mu + \lambda}  \\
&= \frac{k + k\mu + \lambda}{2k\mu } \frac{2k\mu}{k+1+(k-1)\mu + \lambda}  \\
&= \frac{k + k\mu + \lambda}{k + 1 + (k-1)\mu + \lambda} = b_L(k,\mu,\lambda).
\end{align*}
Similarly, the upper bound in \eqref{eq:sdc-bounds-inequalities} follows from Lemma~\ref{lemma:dice-upper-bound}, as
\begin{align*}
    \frac{d_k(\bp)}{\widetilde{d}_k(\bp)}  &\le \frac{1}{\widetilde{d}_k(\bp)} \sum_{i=0}^{\infty} \frac{\lambda^{i} e^{-\lambda}}{i!} \frac{2k\mu}{k+k\mu+i}  \\
    &= \frac{k + k\mu + \lambda}{2k\mu } \sum_{i=0}^{\infty} \frac{\lambda^{i} e^{-\lambda}}{i!} \frac{2k\mu}{k+k\mu+i}  \\
    &= \sum_{i=0}^{\infty} \frac{\lambda^{i} e^{-\lambda}}{i!} \frac{k + k\mu + \lambda}{k+k\mu+i} = b_U(k,\mu,\lambda).
\end{align*}

Finally, statement \eqref{eq:bounds-around-1} can be shown, from the left-hand side, by noting that $\mu\le 1$, and, thus \[
    b_L(k,\mu,\lambda) = \frac{k + k\mu + \lambda}{k + k\mu + \lambda + (1-\mu)} \le 1
\] 
while, from the right-hand side, we have \[
    b_U(k,\mu,\lambda) = \mathbb{E}_{w\sim \text{Poisson}(\lambda)} \left[ \frac{k + k\mu + \lambda}{k+k\mu+w} \right] 
\] which, by applying Jensen's inequality, yields \[
	b_U(k,\mu,\lambda) \ge \frac{k + k\mu + \lambda}{k+k\mu+\mathbb{E}_{w\sim \text{Poisson}(\lambda)} [w] } = \frac{k + k\mu + \lambda}{k+k\mu+\lambda} = 1.
\] 
\end{proof}

\section{On the Definition of the Dice Coefficient}\label{app:dice}

To the best of our knowledge, no previous work has considered the issue of the definition of the Dice coefficient $D(\by, \byh)$ when both $\by=\byh=\bzero$.
Presumably, since we cannot control~$\byh$, the issue has been unnoticed or ignored because the case $\by = \bzero$ is not expected to appear in a meaningful segmentation problem. 
Indeed, we are not aware of any published work that includes samples with $\by = \bzero$ when evaluating the segmentation performance of a model using the Dice coefficient.

The reason, we believe, is that including such negative samples would result in mixing the semantic segmentation problem with that of image-level binary classification (detection), for which the Dice metric is arguably inappropriate. 
Fundamentally, in scenarios where negative cases are conceivable (for instance, when detecting brain tumor), one is first and foremost interested in knowing whether a certain object or feature is present in an image; only when a positive detection is made may one be then interested in a fine-grained segmentation. 
Note that even if a single segmentation model is used to provide both image-level classification and segmentation predictions, it is clear that segmentation quality should only be evaluated on images that actually have a non-empty segmentation mask.

Now, if we can assume that $\by \neq \bzero$, then the issue becomes irrelevant for all practical purposes; however it may still be relevant from a mathematical standpoint. For instance, simply removing $\by=\bzero$ from the domain of $D(\by, \byh)$ would invalidate the conditional independence assumption and complicate proofs that require such independence. Alternatively, we have found that our proposed convention $D(\bzero,\bzero)=0$ leads to nicer properties and cleaner expressions. In particular, it follows from our definition that $D(\by, \bzero) = 0$ for all $\by \in \{0,1\}^n$ and $D(\bzero, \byh) = 0$ for all $\byh \in \{0,1\}^n$, from which \eqref{eq:theo-null-statement} can be concisely stated.

Finally, it is worth mentioning that the definition $D(\bzero, \bzero) = 0$ is implicitly (and perhaps inadvertently) adopted in \citet{dai_rankseg_2023}. While their definition of the Dice coefficient does not address the case $\by=\byh=\bzero$ and the proof of their Theorem~1 overlooks this case (which would lead to a 0/0 division), their definition~(4), when evaluated for $\gamma=0$ and $\tau = 0$ (in their notation), implicitly requires $D(\bzero, \bzero) = 0$ to achieve the desired result. In particular, their result implies $\IDC(\bp, \bzero) = 0$ for all $\bp$, which can only be true if we assume $D(\bzero, \bzero) = 0$.

\section{Semantic Segmentation Tasks}\label{app:tasks}
In the following, we describe the segmentation tasks and the respective models considered, as well as the datasets used for both the in-distribution (ID) and out-of-distribution (OOD) evaluations.
\new{%
In general, the distinction between ID and OOD data typically arises from differences in acquisition protocols, imaging devices, annotation procedures, or patient populations. 
For the Optic cup segmentation task, our choice of OOD dataset is described in Section~\ref{ssec:optic-cup}. For the MSWML and Polyp segmentation tasks, we have not selected the datasets ourselves, but simply used the ID and OOD datasets selected by the authors of the corresponding papers. 
}

For all model and task configurations, we set $\gamma = 0.5$ to obtain the hard predictions~$\byh$. An exception was made for multiple sclerosis white matter lesion segmentation using a baseline 3D U-Net model, where the authors recommend setting $\gamma = 0.35$ to achieve improved overall performance~\citep{malinin_shifts_2022}.

\subsection{Brain tumor segmentation}

Precise segmentation of brain tumors in magnetic resonance images (MRIs) is vital for effective surgical planning and post-operative treatment. The Brain Tumor Segmentation 2020 (BraTS 2020)~\citep{Menze2015BRATS,Bakas2017ScienceData,Bakas2018arXiv} challenge benchmarks state-of-the-art tumor segmentation algorithms using multi-modal, pre-operative brain MRIs~\citep{brats_nn_unet}. The training dataset included 369 3D brain scans of patients diagnosed with either low-grade or high-grade glioma. Each scan comprises multi-modal MRI sequences (T1-weighted, contrast-enhanced T1-weighted, T2-weighted, and T2-Flair) that have been preprocessed and manually annotated by domain experts to delineate tumor regions. The segmentation masks annotate the enhancing tumor, peritumoral edema, and necrotic/non-enhancing regions. In this study, we focus on segmenting the whole tumor region, defined as the union of all annotated regions. We randomly partitioned the dataset, yielding \new{296} images for training and validation and \new{73} images for evaluation.

We trained a 3D nnU-Net model, following the methodology proposed by~\citet{brats_nn_unet}.
To reproduce the performance achieved by the authors, we train the model using all tumor regions.
However, we consider only the whole tumor predictions for our experimental results.


\subsection{Breast Cancer Segmentation}

Breast cancer is a leading cause of mortality among women worldwide. Consequently, accurate segmentation of tumors in ultrasound (US) images is critical for delineating tissue margins during lumpectomy procedures, thereby reducing premature deaths~\citep{al-dhabyani_dataset_2020}. The Breast Ultrasound Images (BUSI) dataset contains 780 images of size 500$\times$500 pixels, collected in 2018 at Baheya Hospital in Egypt from women aged between 25 and 75 years. Although the dataset categorizes images into normal, benign, and malignant classes, we focus solely on the benign and malignant classes because the normal class does not contain tumor structures. We combined 210 malignant tumor images with \newnew{437} benign images, forming a single tumor class comprising \newnew{647} images. We randomly partitioned the data into \newnew{450} images for training, 99 for validation, and 98 for evaluation.

We trained the UNeXt model using the implementation provided at \url{https://github.com/jeya-maria-jose/UNeXt-pytorch}, without modifying the original code. We resized the BUSI images to $256\times256$, and we retained the hyperparameters from~\citet{valanarasu_unext_2022}. Specifically, we used the Adam optimizer with an initial learning rate of $1\times10^{-4}$ and momentum of $0.9$, along with a cosine annealing scheduler that decayed the learning rate to $1\times10^{-5}$. We trained the model for 400 epochs using a mini-batch size of 8, and optimized it using a combined binary cross-entropy and Soft Dice loss.

\subsection{Skin Cancer Segmentation}



Skin cancer is one of the most prevalent cancers globally, and melanoma—known for its high fatality rate among cutaneous malignancies—necessitates early detection to reduce mortality~\citet{codella_skin_2018}. The International Skin Imaging Collaboration (ISIC) organizes an annual challenge comprising multiple tasks. In the 2018 edition, the challenge included lesion segmentation, feature detection, and disease classification; in this study, we focus on lesion segmentation. The original dataset was divided into a training set of 2,594 images, a validation set of 100 images, and a holdout test set of 1,000 images. To ensure that the entire dataset comes from the same distribution, we merged these subsets and randomly repartitioned them into 2,155 images for training, 539 for validation, and 1,000 for evaluation.

We trained the UNeXt architecture using the implementation from \url{https://github.com/jeya-maria-jose/UNeXt-pytorch}. As with the BUSI dataset, we kept the original hyperparameters from~\citet{valanarasu_unext_2022}. We resized all images to $512\times512$ and used the Adam optimizer with an initial learning rate of 0.0001 and a momentum of 0.9. We applied a cosine annealing scheduler to decay the learning rate to 0.00001. The loss function combined 0.5 times binary cross-entropy with Soft Dice loss. We trained the model for 400 epochs with a batch size of 8.

\subsection{Multiple Sclerosis White Matter Lesion (MSWML) Segmentation}



The segmentation of white matter lesions caused by multiple sclerosis helps physicians follow the progress of the disease through the evolution of the lesions~\citep{thompson_diagnosis_2018}. We used the datasets and pretrained baseline segmentation model provided in the Shifts 2.0 challenge~\citep{malinin_shifts_2022}. Specifically, we used the ISBI~\citep{carass_longitudinal_2017} and MSSEG-1~\citep{commowick_objective_2018} datasets as our in-distribution data, while we treated the PubMRI dataset~\citep{lesjak_novel_2018} as out-of-distribution. For ID evaluation, we used the original $\text{Evl}_{\text{in}}$ set from Shifts 2.0, and for OOD evaluation, we used the $\text{Dev}_{\text{out}}$ set. We did not include the $\text{De}_{\text{in}}$ set in our evaluation as it corresponds to the Shifts 2.0 validation set.

We used the pretrained 3D U-Net baseline model available at \url{https://github.com/Shifts-Project/shifts}. The original model was trained using the Adam optimizer~\citep{adam} with an initial learning rate of $1\times10^{-3}$ and a polynomial decay schedule. It used a batch size of 2 and was trained for 15,000 iterations (approximately 30 epochs) with a loss composed of voxel-wise cross-entropy and Soft Dice losses. The model also employed data augmentation, including random rotations, flips, and intensity scaling.

\subsection{Optic Cup Segmentation}\label{ssec:optic-cup}

Accurate segmentation of the optic cup and disc in color fundus photography (CFP) is essential for the early diagnosis of glaucoma and the prevention of irreversible vision loss~\citep{orlando_refuge_2020}. To address this task, we used three publicly available datasets: REFUGE~\citep{orlando_refuge_2020} (1200 labeled images), ORIGA~\citep{zhuo_zhang_origa-light_2010} (650 images), and G1020~\citep{bajwa_g1020_2020} (1020 images). For ID evaluation, we merged all subsets of the REFUGE dataset and randomly divided the 1200 labeled images into 960 for training and 240 for evaluation. We used the entire ORIGA and G1020 datasets for OOD evaluation. \new{These datasets were collected in different hospitals and under different acquisition conditions, including different camera devices and imaging protocols~\citep{bajwa_g1020_2020}.}

Following~\citet{li_medical_2021}, we preprocessed all images and corresponding masks using the MNet DeepCDR model~\citep{fu_disc-aware_2018} to crop a region of interest around the optic disc and cup. We trained the Segtran model~\citep{li_medical_2021}, following the implementation and hyperparameters described by the authors. Specifically, we used the AdamW optimizer~\citep{adamw} with a learning rate of $2\times10^{-4}$ and trained the model for 10{,}000 iterations (approximately 27 epochs) with a batch size of 4. We used the average of pixel-wise cross-entropy and Soft Dice loss functions for training. The Segtran architecture included three transformer layers with Expanded Attention Blocks using 4 modes ($N_m=4$) and EfficientNet-B4 as the backbone. Additional configuration details are available at \url{https://github.com/askerlee/segtran}.

\subsection{Polyp Segmentation}

Polyp segmentation in colonoscopy images facilitates the identification and assessment of polyps, a crucial step in preventing colorectal cancer~\citep{jha_kvasir-seg_2019}. In this study, we constructed the polyp segmentation dataset as described in \citet{dong_polyp-pvt_2023}. The training set comprises 900 images from KvasirSEG~\citep{jha_kvasir-seg_2019} and 548 images from ClinicDB~\citep{ClinicDB}, while the ID evaluation set includes the remaining 100 images from KvasirSEG and 64 from ClinicDB. The OOD dataset contains 196 images from ETIS~\citep{silva_toward_2014}, 380 from ColonDB~\citep{tajbakhsh_automated_2016}, and 60 from EndoScene~\citep{vazquez_benchmark_2017}.

For polyp segmentation we employed the results from the Polyp-PVT model, a deep learning architecture based on the pyramid vision transformer (PVT). We did not train the Polyp-PVT model; instead, we used the precomputed predictions and corresponding ground truth masks provided at \url{https://github.com/DengPingFan/Polyp-PVT}. The architecture, weights, and predictions are described in detail in~\citet{dong_polyp-pvt_2023}.


\vskip 0.2in
\bibliography{references,newbib}

\end{document}